%% file: main.tex
\documentclass[final]{styles/colt2021/colt2021} 

\input{header}
\title[Impossible Tuning Made Possible]{Impossible Tuning Made Possible: \\ A New Expert Algorithm and Its Applications}
\usepackage{times}
\coltauthor{%
 \Name{Liyu Chen\nametag{\thanks{Authors are listed in alphabetical order.}}} \Email{liyuc@usc.edu}\\
 \addr University of Southern California
 \AND
 \Name{Haipeng Luo\nametag{\footnotemark[1]}} \Email{haipengl@usc.edu}\\
 \addr University of Southern California
 \AND
 \Name{Chen-Yu Wei\nametag{\footnotemark[1]}} \Email{chenyu.wei@usc.edu}\\
 \addr University of Southern California
}

\begin{document}
\SetAlgoVlined
\DontPrintSemicolon
\maketitle

\begin{abstract}
We resolve the long-standing ``impossible tuning'' issue for the classic expert problem and show that, it is in fact possible to achieve regret $\otil\Big(\sqrt{(\ln d)\sumt \ell_{t,i}^2}\Big)$ simultaneously for all expert $i$ in a $T$-round $d$-expert problem where $\ell_{t,i}$ is the loss for expert $i$ in round $t$.
Our algorithm is based on the Mirror Descent framework with a correction term and a weighted entropy regularizer.
While natural, the algorithm has not been studied before and requires a careful analysis.
We also generalize the bound to $\otil\Big(\sqrt{(\ln d)\sumt (\ell_{t,i}-m_{t,i})^2}\Big)$ for any prediction vector $m_t$ that the learner receives,
and recover or improve many existing results by choosing different $m_t$.
Furthermore, we use the same framework to create a master algorithm that combines a set of base algorithms and learns the best one with little overhead.
The new guarantee of our master allows us to derive many new results for both the expert problem and more generally Online Linear Optimization.
\end{abstract}

\section{Introduction}
\input{intro}

\section{An Algorithmic Framework}\label{sec:framework}
\input{framework}

\section{Applications to the Expert Problem}\label{sec:app_expert}
\input{app_expert}

\section{Applications to Online Linear Optimization}\label{sec:app_OLO}
\input{app_OLO}

\section{Discussions and Open Problems}
We mention two open questions for the expert problem.
First, in the case when we are required to select one expert $i_t$ randomly in each round $t$, and the regret against $i$ is measured by $\sumt \ell_{t,i_t} -\ell_{t,i}$,
it is unclear how to achieve our bounds such as $\tilO{\sqrt{(\ln d) \sumt \ell_{t,i}^2}}$ with high probability (even though our results clearly imply this in expectation).
The difficulty lies in handling the deviation between $\sumt\inner{w_t}{\ell_t}$ and $\sumt \ell_{t, i_t}$ and bounding it in terms of only $\sumt\ell_{t,i}^2$.
We conjecture that impossible tuning might indeed be impossible in this case.

Second, note that even though we only focus on having one prediction sequence $\{m_t\}_{t\in [T]}$, we can in fact also deal with multiple sequences and learn the best via another expert algorithm, similarly to~\citep{rakhlin2013online}.
One caveat is that the trick we apply in Implications 2-4 of \pref{subsec:impossible_tuning} (that $m_t$ can depend on $\ell_t$ even though it is unknown) does not work anymore, since different experts might be using different sources of predictions and thus the calculation of $w_t$ does require knowing all predictions at the beginning of round $t$.
Due to this issue, we for example cannot achieve a bound in the form of 
\[
\forall i\in[d], \;\; \reg(e_i) = \tilO{\sqrt{(\ln d)\min\cbr{\sumt \ell_{t,i}^2, \sumt (\ell_{t,i} - \inner{w_{t}}{\ell_t})^2}}}.
\]
We leave the possibility of achieving such a bound as an open problem.

\acks{This work is supported by NSF Award IIS-1943607.}

\bibliography{ref}
\newpage

\appendix

\section{Useful Lemmas Related to OMD}
\input{appendix_lemma}

\section{Omitted Details for \pref{sec:framework}}\label{app:framework}
\input{appendix_framework}

\section{Omitted Details for \pref{sec:app_expert}}\label{app:expert}
\input{appendix_expert}

\section{Omitted Details for \pref{sec:app_OLO}}\label{app:OLO}
\input{appendix_OLO}

\end{document}

%% file: header.tex
\usepackage{times}
\usepackage{booktabs}       
\usepackage{amsfonts}       
\usepackage{nicefrac}       
\usepackage{enumitem}
\usepackage{nicefrac}

\usepackage{amsmath,amssymb,mathrsfs}

\usepackage{footnote}
\usepackage{algorithm}
\usepackage{algorithmic}

\allowdisplaybreaks

\makeatletter
\def\set@curr@file#1{\def\@curr@file{#1}} 
\makeatother

\usepackage{graphicx}
\usepackage{mathtools}
\usepackage{footnote}
\usepackage{float}
\usepackage{xspace}
\usepackage{multirow}
\usepackage{wrapfig}
\usepackage{framed}
\usepackage{bbm}
\usepackage{footnote}
\usepackage{nicefrac}
\usepackage{makecell}
\usepackage[T1]{fontenc}
\makesavenoteenv{tabular}
\makesavenoteenv{table}

\definecolor{Green}{rgb}{0.13, 0.65, 0.3}
\usepackage[]{color-edits}

\addauthor{HL}{red}
\addauthor{LC}{cyan}
\addauthor{CW}{green}

\newcommand{\calA}{{\mathcal{A}}}
\newcommand{\calB}{{\mathcal{B}}}

\newcommand{\calD}{{\mathcal{D}}}
\newcommand{\calE}{{\mathcal{E}}}

\newcommand{\calG}{{\mathcal{G}}}
\newcommand{\calI}{{\mathcal{I}}}

\newcommand{\calK}{{\mathcal{K}}}
\newcommand{\calL}{{\mathcal{L}}}

\newcommand{\calP}{{\mathcal{P}}}

\newcommand{\calS}{{\mathcal{S}}}

\newcommand{\calZ}{{\mathcal{Z}}}

\newcommand{\one}{\boldsymbol{1}}
\newcommand{\zero}{\boldsymbol{0}}

\newcommand{\KL}{\text{\rm KL}}

\DeclareMathOperator*{\argmin}{argmin}
\DeclareMathOperator*{\argmax}{argmax}
\DeclarePairedDelimiter\abs{\lvert}{\rvert}

\newcommand{\eat}[1]{}

\newcommand{\inner}[2]{\left\langle #1,#2 \right\rangle}
\newcommand{\inners}[2]{\langle #1,#2\rangle}

\newcommand{\rbr}[1]{\left(#1\right)}
\newcommand{\sbr}[1]{\left[#1\right]}
\newcommand{\cbr}[1]{\left\{#1\right\}}

\newcommand{\abr}[1]{\left|#1\right|}
\newcommand{\bigO}[1]{\order\left( #1 \right)}
\newcommand{\tilO}[1]{\otil\left( #1 \right)}

\newcommand{\bigo}[1]{\order( #1 )}
\newcommand{\tilo}[1]{\otil( #1 )}

\DeclarePairedDelimiter\ceil{\lceil}{\rceil}

\newcommand{\mgeq}{\succcurlyeq}
\newcommand{\mleq}{\preccurlyeq}

\newcommand{\cmin}{\ensuremath{c_{\min}}}

\newcommand{\reg}{\textsc{Reg}}

\newcommand{\lgT}{\lceil \log_2 T\rceil}
\newcommand{\cmax}{c_{\max}}

\newcommand{\tilB}{\widetilde{B}}

\newcommand{\wstar}{w^{\star}}
\newcommand{\mss}{\calS_{\text{\rm MS}}}
\newcommand{\ONSbase}{\calE_{\text{\rm ONS}}}
\newcommand{\GDbase}{\calE_{\text{\rm GD}}}
\newcommand{\AdaGradbase}{\calE_{\text{\rm AG}}}
\newcommand{\MetaGradbase}{\calE_{\text{\rm MG}}}
\newcommand{\KLbase}{\calE_{\text{\rm KL}}}
\newcommand{\URbase}{\calE_{\text{\rm UR}}}
\newcommand{\switchbase}{\calE_{\text{\rm switch}}}
\newcommand{\msbase}{\calE_{\text{\rm MS}}}
\newcommand{\OMLProd}{\textsc{Optimistic-Adapt-ML-Prod}\xspace}
\newcommand{\MLProd}{\textsc{Adapt-ML-Prod}\xspace}
\newcommand{\ABProd}{$(\calA, \calB)$-\textsc{Prod}\xspace}
\newcommand{\Prod}{\textsc{Prod}\xspace}
\newcommand{\istar}{{i_\star}}
\newcommand{\kstar}{{k_\star}}
\newcommand{\etastar}{\eta_\star}
\newcommand{\Astar}{\calA_\star}

\newcommand{\master}{\textsc{MsMwC-Master}\xspace}
\newcommand{\alg}{\textsc{MsMwC}\xspace}
\newcommand{\msmw}{\textsc{MsMw}\xspace}
\newcommand{\kl}{f_{\text{\rm KL}}}
\newcommand{\simplex}[1][d]{\Delta_{#1}}
\newcommand{\sumj}{\sum_{j=1}^S}
\newcommand{\sumi}{\sum_{i=1}^d}
\newcommand{\sumt}{\sum_{t=1}^T}

\newcommand{\REG}{\textrm{REG}}

\newcommand{\field}[1]{\mathbb{#1}}

\newcommand{\fR}{\field{R}}

\newcommand{\fZ}{\field{Z}}
\newcommand{\E}{\field{E}}

\newcommand{\norm}[1]{\left\|{#1}\right\|}
\newcommand{\dnorm}[1]{\left\|{#1}\right\|_\star}
\newcommand{\trace}[1]{\text{\rm tr}\left({#1}\right)}

\newcommand{\defeq}{\stackrel{\rm def}{=}}

\newcommand{\order}{\ensuremath{\mathcal{O}}}

\newcommand{\otil}{\ensuremath{\tilde{\mathcal{O}}}}

\usepackage{prettyref}
\newcommand{\pref}[1]{\prettyref{#1}}
\newcommand{\pfref}[1]{Proof of \prettyref{#1}}
\newcommand{\savehyperref}[2]{\texorpdfstring{\hyperref[#1]{#2}}{#2}}
\newrefformat{eq}{\savehyperref{#1}{Eq.~\textup{(\ref*{#1})}}}
\newrefformat{eqn}{\savehyperref{#1}{Equation~\ref*{#1}}}
\newrefformat{lem}{\savehyperref{#1}{Lemma~\ref*{#1}}}
\newrefformat{def}{\savehyperref{#1}{Definition~\ref*{#1}}}
\newrefformat{line}{\savehyperref{#1}{Line~\ref*{#1}}}
\newrefformat{thm}{\savehyperref{#1}{Theorem~\ref*{#1}}}
\newrefformat{corr}{\savehyperref{#1}{Corollary~\ref*{#1}}}
\newrefformat{cor}{\savehyperref{#1}{Corollary~\ref*{#1}}}
\newrefformat{sec}{\savehyperref{#1}{Section~\ref*{#1}}}
\newrefformat{subsec}{\savehyperref{#1}{Section~\ref*{#1}}}
\newrefformat{app}{\savehyperref{#1}{Appendix~\ref*{#1}}}
\newrefformat{assum}{\savehyperref{#1}{Assumption~\ref*{#1}}}
\newrefformat{ex}{\savehyperref{#1}{Example~\ref*{#1}}}
\newrefformat{fig}{\savehyperref{#1}{Figure~\ref*{#1}}}
\newrefformat{alg}{\savehyperref{#1}{Algorithm~\ref*{#1}}}
\newrefformat{rem}{\savehyperref{#1}{Remark~\ref*{#1}}}
\newrefformat{conj}{\savehyperref{#1}{Conjecture~\ref*{#1}}}
\newrefformat{prop}{\savehyperref{#1}{Proposition~\ref*{#1}}}
\newrefformat{proto}{\savehyperref{#1}{Protocol~\ref*{#1}}}
\newrefformat{prob}{\savehyperref{#1}{Problem~\ref*{#1}}}
\newrefformat{claim}{\savehyperref{#1}{Claim~\ref*{#1}}}
\newrefformat{que}{\savehyperref{#1}{Question~\ref*{#1}}}
\newrefformat{op}{\savehyperref{#1}{Open Problem~\ref*{#1}}}
\newrefformat{fn}{\savehyperref{#1}{Footnote~\ref*{#1}}}
\newrefformat{tab}{\savehyperref{#1}{Table~\ref*{#1}}}
\newrefformat{fig}{\savehyperref{#1}{Figure~\ref*{#1}}}

\DeclareOldFontCommand{\rm}{\normalfont\rmfamily}{\mathrm}
\DeclareOldFontCommand{\it}{\normalfont\itshape}{\mathit}

%% file: intro.tex

In the classic expert problem~\citep{freund1997decision}, a learner interacts with an adversary for $T$ rounds,
where in each round $t$, the learner first decides a distribution $w_t \in \simplex$ over a fixed set of $d$ experts, and then the adversary decides a loss vector $\ell_t \in \fR^d$.
The learner suffers loss $\inner{w_t}{\ell_t}$ and observes $\ell_t$ at the end of round $t$. The regret against a fixed strategy $u \in \simplex$ is defined as $\reg(u) = \sumt \inner{w_t - u}{\ell_t}$.
Many existing algorithms achieve $\max_u \reg(u) = \max_i \reg(e_i) = \bigo{\sqrt{T\ln d}}$, which is known to be minimax optimal.

In particular, both the \Prod algorithm~\citep{cesa2007improved}, which sets $w_{t+1,i} \propto w_{t, i}(1-\eta \ell_{t,i})$, and a variant of the classic multiplicative-weight~\citep{steinhardt2014adaptivity}, which sets $w_{t+1,i} \propto w_{t, i}e^{-\eta\ell_{t,i} - \eta^2 \ell_{t,i}^2}$, achieve a regret bound $\reg(e_i) \leq \frac{\ln d}{\eta} + \eta\sumt \ell_{t,i}^2$ for some learning rate $\eta$.
With the optimal tuning of $\eta$, this gives an adaptive bound $\reg(e_i) = \order\Big(\sqrt{(\ln d) \sumt \ell_{t,i}^2}\Big)$, potentially much better than the minimax bound.
However, since different expert $i$ requires a different tuning, no method is known to achieve this bound {\it simultaneously for all $i$}.
Several works discuss the difficulty of doing so even with different $\eta$ for different experts and why all standard tuning techniques fail~\citep{cesa2007improved, hazan2010extracting}.
Indeed, the problem is so challenging that it has been referred to as the ``impossible tuning'' issue~\citep{gaillard2014second}.

Our first main contribution is to show that, perhaps surprisingly, this impossible tuning is in fact {\it possible} (up to an additional $\ln T$ factor), via an algorithm combining ideas that mostly appear before already.
More concretely, we achieve this via Mirror Descent with a correction term similar to~\citep{steinhardt2014adaptivity} and a weighted negative entropy regularizer with different learning rates for each expert (and each round) similar to~\citep{bubeck2017online}.
Note that while natural, this algorithm has not been studied before,\footnote{%
Except that a simpler version is used in a concurrent work~\citep{chen2020minimax} by the same authors for a different problem (learning stochastic shortest path).
}
 and is {\it not equivalent} to using different learning rates for different experts in \Prod or multiplicative-weight, as it does not admit a closed ``proportional'' form (and instead needs to be computed via a line search).
Crucially, our analysis carefully utilizes a negative term in the regret bound to achieve the claimed result.

\renewcommand{\arraystretch}{1.2}
\begin{table}[t]
	\centering
	\caption{Summary of main results. $w_t\in\fR^d$ is the decision of the learner, $\ell_t$ is the loss vector, $m_t$ is a prediction for $\ell_t$, $\calL_T = \sumt (\ell_t-m_t)(\ell_t-m_t)^\top$, and $r$ is the rank of $\calL_T$.}
	\label{tab:main}
	\resizebox{.94\textwidth}{!}{%
	\begin{tabular}{|l|l|l|}
		\hline
		 & \multicolumn{1}{c|}{Results} & \multicolumn{1}{c|}{Notes} \\
		 \hline
		 \multirow{5}{*}{\makecell{Expert \\ $\reg(e_i)$}} & \multicolumn{1}{c|}{$\otil\Big(\sqrt{(\ln d)\sumt (\ell_{t,i}-m_{t,i})^2}\Big)$}   & \multirow{5}{*}{\makecell[l]{$\bullet$\; $\ln d$ can be generalized to $\KL(u, \pi)$ \\ \quad for competitor $u$ and prior $\pi$ \\ $\bullet$\; all results generalize to switching \\ \quad regret and unknown loss range \\ $\bullet$\; analogue for interval regret or bandits \\ \quad is impossible }} \\
		 & \multicolumn{1}{l|}{With different $m_t$, $(\ell_{t,i}-m_{t,i})^2$ becomes:} &   \\
		 & \multicolumn{1}{l|}{$\bullet$\;$\ell_{t,i}^2$ \hspace{1.8cm} $\bullet$\;$(\ell_{t,i}-\frac{1}{T}\sum_{s=1}^T \ell_{s,i})^2$} & \\
		 & \multicolumn{1}{l|}{$\bullet$\;$(\ell_{t,i}-\ell_{t-1,i})^2$ \hspace{0.11cm}$\bullet$\;$(\ell_{t,i}-\ell_{t,1})^2$} &  \\
		 & \multicolumn{1}{l|}{$\bullet$\;$\inner{w_t-e_i}{\ell_t}^2$ \hspace{0.18cm} $\bullet$\;$\inner{w_t-e_i}{\ell_t-m_t}^2$} &  \\
		 \hline
		 \multirow{5}{*}{\makecell{OLO \\ $\reg(u)$}} & \multicolumn{1}{l|}{$\otil\Big(\sqrt{r\sumt\inner{u}{\ell_t-m_t}^2}\Big)$}  & \multirow{3}{*}{\makecell{$\bullet$\; first three bounds hold simultaneously}} \\
		 & \multicolumn{1}{l|}{$\otil\Big(\norm{u}\sqrt{\sumt\dnorm{\ell_t-m_t}^2}\Big)$} &  \\
		 & \multicolumn{1}{l|}{$\otil\Big(\sqrt{\big(\norm{u}_2^2+u^\top\calL_T^{1/2}u \big)\text{\rm tr}\big(\calL_T^{1/2}\big)}\Big)$} &  \multirow{2}{*}{\makecell[l]{$\bullet$\; all results generalize to unconstrained \\ \quad learning and unknown Lipschitzness}}\\
		 & \multicolumn{1}{l|}{$\otil\Big(\sqrt{r\sumt\inner{u-w_t}{\ell_t-m_t}^2}\Big)$} &  \\
		 \hline
	\end{tabular}
	}
\end{table}

\tolerance 1414
\hbadness 1414
\emergencystretch 1.5em
\hfuzz 0.3pt
\widowpenalty=10000
\vfuzz \hfuzz
\raggedbottom

We present our result in a more general setting where the learner receives a predicted loss vector $m_t$ before deciding $w_t$~\citep{rakhlin2013optimization}, and show a bound $\reg(e_i)=\otil\Big(\sqrt{(\ln d)\sumt (\ell_{t,i}-m_{t,i})^2}\Big)$ simultaneously for all $i$
(setting $m_t=0$ resolves the original impossible tuning issue).
Using different $m_t$, we achieve various regret bounds summarized in \pref{tab:main}, which either recover the guarantees of existing algorithms such as \ABProd~\citep{sani2014exploiting}, \MLProd~\citep{gaillard2014second}, \OMLProd~\citep{wei2016tracking}, or improve over existing variance/path-length bounds in~\citep{steinhardt2014adaptivity}. 
We also show that the bound $\otil\Big(\sqrt{(\ln d)\sumt \inner{w_t-e_i}{\ell_t-\ell_{t-1}}^2}\Big)$, obtained by \citep{wei2016tracking} and our work, simultaneously ensures the ``fast rate'' consequences discussed in~\citep{koolen2016combining} for stochastic settings and the path-length bound useful for fast convergence in games~\citep{syrgkanis2015fast}.
See \pref{subsec:impossible_tuning} for detailed discussions.

Our second main contribution is to use the same algorithmic framework to create a master algorithm that combines a set of base algorithms and learns the best for different environments (\pref{subsec:master}).
Although similar ideas appear in many prior works with different masters~\citep{koolen2014learning, van2016metagrad, foster2017parameter, cutkosky2019combining, bhaskara2020online},
the new guarantee of our master allows us to derive many new results that cannot be achieved before, for both the expert problem and more generally Online Linear Optimization (OLO).

Specifically, for the expert problem, 
using the master to combine different instances of {\it itself}, 
we further generalize the aforementioned bound from different aspects,
including replacing the $\ln d$ factor with $\KL(u,\pi)$ when competing against $u$ with a prior distribution $\pi$,
adapting to the scale of each expert,
extending the results to switching regret, and dealing with unknown loss range.
These results improve over~\citep{luo2015achieving, koolen2015second}, \citep{bubeck2017online,foster2017parameter, cutkosky2018black}, \citep{cesa2012mirror}, and~\citep{mhammedi2019lipschitz} respectively.
See \pref{sec:app_expert} for detailed discussions.

Next, we consider the more general OLO problem where the learner's decision set generalizes from $\simplex$ to an arbitrary closed convex set $\calK \subset \fR^d$
(other than this change, the learning protocol and the regret definition remain the same).
Using our master to combine different types of base algorithms, we achieve four different and incomparable bounds on $\reg(u)$ simultaneously for all $u$, listed in \pref{tab:main}.
Importantly, the first three bounds can be achieved at the same time {\it with one single algorithm}.
These bounds improves over a line of recent advances in OLO~\citep{van2016metagrad, cutkosky2018black, cutkosky2019artificial, cutkosky2019combining, mhammedi2019lipschitz, mhammedi2020lipschitz, cutkosky2020better}.
See \pref{sec:app_OLO} for detailed discussions.


\paragraph{Notation}
Throughout the paper,
$\simplex$ denotes the $d-1$ dimensional simplex; $e_i, \zero, \one \in \fR^d$ are respectively the $i$-th standard basis vector, the all-zero vector, and the all-one vector;
$[n]$ denotes the set $\{1, \ldots, n\}$;
$\KL(\cdot, \cdot)$ denotes the KL divergence;
$\norm{u}_A=\sqrt{u^\top A u}$ is the quadratic norm with respect to a matrix $A$;
$D_{\psi}(u, w)=\psi(u)-\psi(w)-\inner{\nabla\psi(w)}{u-w}$ is the Bregman divergence of $u$ and $w$ with respect to a convex function $\psi$,
and $\tilo{\cdot}$ hides logarithmic dependence on $T$.

%% file: framework.tex

Consider the expert problem and recall that the learner sequentially decides a distribution $w_t \in \Delta_d$ (with the help of a prediction $m_t \in \fR^d$) and then observes the loss vector $\ell_t \in \fR^d$.
Note that we do not make the typical assumption $\ell_{t,i}\in [0,1]$ or $|\ell_{t,i}|\leq 1$; instead, the requirement (if any) on the range of the losses will be stated either explicitly or implicitly in the conditions of each lemma or theorem.

We start by proposing a general algorithmic framework called Multi-scale Multiplicative-weight with Correction (\alg), shown in \pref{alg:framework}.
In \pref{subsec:impossible_tuning}, we instantiate the framework in a specific way to resolve the impossible tuning issue, and in \pref{subsec:master}, we instantiate it differently to obtain a new master algorithm, with more applications discussed in following sections.

\alg is a variation of the standard Optimistic-Mirror-Descent (OMD) framework, which maintains two sequences $w_1, \ldots, w_T$ and $w_1', \ldots, w_T'$ updated according to \pref{line:OMD1} and \pref{line:OMD2}.
The key new ingredients are the following.
First, we adopt a time-varying decision subset $\Omega_t \subseteq \simplex$ to which $w_t$ and $w_{t+1}'$ belong.
This is decided at the beginning of each round $t$ and is useful for applications discussed in \pref{subsec:unknown_range} and \pref{app:unconstrained}, where we need to eliminate some experts on-the-fly.
(For other applications, $\Omega_t$ is either $\Delta_d$ or its truncated version throughout all $T$ rounds.)

Second, our regularizer $\psi_t(w) = \sum_{i=1}^d \frac{1}{\eta_{t, i}}w_i\ln w_i$ is negative entropy with individual and time-varying learning rate $\eta_{t,i}$ for each expert $i$.
For most applications, $\eta_{t,i}$ is the same for all $t$, in which case our regularizer is the same as that used in the \msmw algorithm of~\citep{bubeck2017online}.

Finally, we adopt a second-order correction term $a_t$ added to the loss vector $\ell_t$ in the update of $w_{t+1}'$ (\pref{line:OMD2}), which is the most important difference compared to \msmw~\citep{bubeck2017online}.
Similar correction terms have been used in prior works such as~\citep{hazan2010extracting, steinhardt2014adaptivity, wei2018more} and are known to be important to achieving a regret bound that depends on quantities only related to the expert being compared to.

One can see that essentially all ingredients of \alg appear before in the literature.
However, the specific combination of these ingredients (which has not been studied before) and a careful analysis enable us to resolve the impossible tuning issue as well as developing other new results.

\begin{algorithm}[t]
\caption{Multi-scale Multiplicative-weight with Correction (\alg)}
\label{alg:framework}
	\textbf{Initialize:} $w_1' \in \simplex$. 
		
	\For{$t=1,\ldots,T$}{
	     \nl Receive prediction $m_t \in \fR^d$.
	     
	     \nl Decide a compact convex decision subset $\Omega_t \subseteq \simplex$ and learning rates $\eta_{t} \in \fR^d_{\geq 0}$. 
	
	     \nl Compute $w_t = \argmin_{w\in\Omega_t}\inner{w}{m_t} + D_{\psi_t}(w, w_t')$ where $\psi_t(w) = \sum_{i=1}^d \frac{1}{\eta_{t, i}}w_i\ln w_i$. \label{line:OMD1}
	     
	     \nl Play $w_t$, receive $\ell_t$, and construct correction term $a_t \in \fR^d$ with $a_{t,i} = 32\eta_{t,i}(\ell_{t,i}-m_{t,i})^2$.
	     
		\nl Compute $w_{t+1}' = \argmin_{w\in\Omega_t}\inner{w}{\ell_t+a_t} + D_{\psi_t}(w, w_t')$. \label{line:OMD2}
	}
\end{algorithm}

We present a general lemma on the regret guarantee of \alg below, which holds under a condition on the magnitude of $\eta_{t,i}|\ell_{t,i}-m_{t,i}|$; see \pref{app:framework} for the proof.
We also note that the last negative term in the regret bound is particularly important for some of the applications.

\begin{lemma}\label{lem:framework}
Define $\kl(a,b) = a\ln\frac{a}{b} - a + b$ for $a, b \in [0,1]$.\footnote{Define $\kl(0,b)=b$ for all $b\in [0,1]$.}
Suppose that for all $t\in[T]$, $32\eta_{t,i}|\ell_{t,i}-m_{t,i}| \leq 1$ holds for all $i$ such that $w_{t,i}>0$. 
Then \alg ensures for any $u \in \bigcap_{t=1}^T \Omega_t$,
\begin{equation}\label{eq:general_bound}
\begin{split}
\reg(u) &\leq \sumi \frac{1}{\eta_{1,i}} \kl(u_i, w_{1,i}')
 + \sum_{t=2}^T \sumi \rbr{\frac{1}{\eta_{t,i}}-\frac{1}{\eta_{t-1,i}}}\kl(u_i, w_{t,i}') \\
&\qquad\qquad + 32 \sum_{t=1}^T\sumi \eta_{t,i} u_i (\ell_{t,i} - m_{t,i})^2
 - 16  \sum_{t=1}^T\sumi \eta_{t,i}w_{t,i}(\ell_{t,i} - m_{t,i})^2.
\end{split}
\end{equation}
\end{lemma}

\subsection{Impossible tuning made possible}\label{subsec:impossible_tuning}

To resolve the impossible tuning issue, we instantiate \alg in the following way with the decision sets fixed to a truncated simplex and the learning rates tuned using data observed so far.

\begin{theorem}\label{thm:impossible_tuning}
Suppose $|\ell_{t,i}|$ and $|m_{t,i}|$ are bounded by $1$ for all $t \in [T]$ and $i \in [d]$.
Then \alg with $w_1'=\frac{1}{d}\one$, $\Omega_1 = \cdots = \Omega_T  = \{w \in \Delta_d: w_i \geq \frac{1}{dT}\}$, and $\eta_{t,i} = \min\cbr{\sqrt{\frac{\ln(dT)}{\sum_{s<t} (\ell_{s,i}-m_{s,i})^2}}, \frac{1}{64}}$ ensures for all $\istar \in [d]$,
$
\reg(e_\istar) = \order\Big(\ln(dT) + \sqrt{\ln(dT)\sumt (\ell_{t,\istar} - m_{t,\istar})^2}\Big)
$.
\end{theorem}
\begin{proof}[sketch]
We apply \pref{eq:general_bound} with $u = (1-\frac{1}{T})e_\istar + \frac{1}{T}w_1' \in \bigcap_{t=1}^T \Omega_t$, so that $\reg(e_\istar) \leq \reg(u) + 2$.
Most calculation is straightforward, and the most important part is to realize that $\rbr{\frac{1}{\eta_{t,i}}-\frac{1}{\eta_{t-1,i}}}w_{t,i}'$, a term from $\rbr{\frac{1}{\eta_{t,i}}-\frac{1}{\eta_{t-1,i}}}\kl(u_i, w_{t,i}')$, can be bounded as:
\[
\rbr{\frac{1}{\eta_{t,i}}-\frac{1}{\eta_{t-1,i}}}w_{t,i}'
= \frac{\frac{1}{\eta_{t,i}^2}-\frac{1}{\eta_{t-1,i}^2}}{\frac{1}{\eta_{t,i}}+\frac{1}{\eta_{t-1,i}}}w_{t,i}'
\leq \eta_{t-1,i}w_{t,i}'\rbr{\frac{1}{\eta_{t,i}^2}-\frac{1}{\eta_{t-1,i}^2}},
\]
which is further bounded by $\frac{1}{\ln(dT)}\eta_{t-1,i}w_{t,i}'(\ell_{t-1,i} - m_{t-1,i})^2$ using the definition of $\eta_{t,i}$, and thus can
be canceled by the last negative term in \pref{eq:general_bound} (since $w_{t,i}'$ and $w_{t-1,i}$ are close).
The complete proof can be found in \pref{app:framework}.
\end{proof}

When $m_t=0$, our bound exactly resolves the original impossible tuning issue (up to a $\ln T$ term). 
Below we discuss more implications of our bound by choosing different $m_t$.


\paragraph{Implication 1: improved variance or path-length bounds.}
Similarly to~\citep{steinhardt2014adaptivity}, by setting $m_t$ to be the running average of the loss vectors $\frac{1}{t-1}\sum_{s<t}\ell_s$, we obtain a bound that depends only on the variance of expert $\istar$: $\order\Big(\sqrt{\ln(dT)\sumt (\ell_{t,\istar} - \mu_{\istar})^2}\Big)$ where $\mu_\istar = \frac{1}{T}\sumt \ell_{t, \istar}$.
On the other hand, by setting $m_t = \ell_{t-1}$ (define $\ell_0 = \zero$), we obtain a bound that depends only on the ``path-length'' of expert $\istar$: $\order\Big(\sqrt{\ln(dT)\sumt (\ell_{t,\istar} - \ell_{t-1,\istar})^2}\Big)$.
The algorithm of~\citep{steinhardt2014adaptivity} uses a fixed learning rate and only achieves these bounds with an oracle tuning of the fixed learning rate, while our algorithm is completely adaptive and parameter-free. \\

In the next few implications, we 
make use of a trick similar to~\citep{wei2018more}: if all coordinates of $m_t$ are the same, then $\inner{w}{m_t}$ is a constant independent of $w \in \simplex$ and thus $w_t = \argmin_{w\in\Omega_t}\inner{w}{m_t} + D_{\psi_t}(w, w_t') = \argmin_{w\in\Omega_t} D_{\psi_t}(w, w_t')$,
meaning that the algorithm and its guarantee are valid {\it even if $m_t$ is set in terms of $\ell_t$} which is unknown at the beginning of round $t$.

\paragraph{Implication 2: recovering \ABProd guarantee.}
If we set $m_t = \ell_{t,1}\one$, then the regret against expert $1$ becomes a constant $\bigO{\ln(dT)}$ (while the regret against others remains $\bigo{\sqrt{T\ln(dT)}}$).
This is exactly the guarantee of the \ABProd algorithm~\citep{sani2014exploiting}, useful for combining a set of base algorithms where one of them enjoys a regret bound significantly better than $\sqrt{T}$.

\paragraph{Implication 3: recovering \MLProd guarantee.}
Next, we set $m_t = \inner{w_t}{\ell_t}\one$ (again, valid even if unknown at the beginning of round $t$),
leading to a bound $\order\Big(\sqrt{\ln(dT)\sumt r_{t,\istar}^2}\Big)$ where $r_{t,i} = \inner{w_t - e_i}{\ell_t}$ is the instantaneous regret to expert $i$.
A regret bound in terms of $\sqrt{\sumt r_{t,\istar}^2}$ is first achieved by the  \MLProd algorithm~\citep{gaillard2014second} (and later improved in~\citep{koolen2015second, wintenberger2017optimal}),
and it has important consequences in achieving fast rates in stochastic settings; see~\citep{koolen2016combining} for in-depth discussions.

\paragraph{Implication 4: recovering \OMLProd guarantee.}
By the same reason, it is also valid to set $m_t = m_t' + \inner{w_t}{\ell_t - m_t'}\one$ for some prediction $m_t' \in [-1,+1]^d$ received at the beginning of round $t$.\footnote{%
This is because $w_t = \argmin_{w\in\Omega_t}\inner{w}{m_t} + D_{\psi_t}(w, w_t') = \argmin_{w\in\Omega_t} \inner{w}{m_t'} + D_{\psi_t}(w, w_t')$.
One caveat is that $m_{t,i}$ is now in the range of $[-3,+3]$, breaking the condition of \pref{thm:impossible_tuning}, but this can be simply addressed by changing the constant $64$ in the definition of $\eta_{t,i}$ to $128$ so that the condition of \pref{lem:framework} still holds.
}
Doing so leads to a bound $\order\Big(\sqrt{\ln(dT)\sumt r_{t,\istar}^{'2}}\Big)$ where $r_{t,i}' = \inner{w_t - e_i}{\ell_t - m_t'}$ is the instantaneous regret to expert $i$ measured with respect to the prediction difference $\ell_t - m_t'$.  This bound first appears in \OMLProd \citep{wei2016tracking} under the special choice of $m_t'=\ell_{t-1}$. 
In the following, we show that this bound preserves the fast rate consequences of the vanilla \MLProd guarantee~\citep{gaillard2014second} (especially when $m_t'$ is set to $\ell_{t-1}$) in stochastic settings, while improving upon it whenever the predictions are accurate.

\begin{theorem}\label{thm:fast_rates}
Suppose that $\ell_1, \ldots, \ell_T$ are generated randomly,
and let $\E_t$ denote the conditional expectation given $\ell_1, \ldots, \ell_{t-1}$.
Then the algorithm described in Implication 4 satisfies the following:
\begin{itemize}
\item
If there exist $\Delta > 0$ and $\istar$ such that $\E_t[\ell_{t,i}-\ell_{t,\istar}] \geq \Delta$ for all $t$ and $i \neq \istar$, then with any $m_t' \in [-1,+1]^d$, $\reg(e_\istar) = \bigO{\frac{\ln(dT)}{\Delta}}$ holds both in expectation and with high probability.

\item
If there exist $\kappa \in [0,1]$, $\Delta > 0$ and $\istar$ such that $\E_t[\ell_{t,i}-\ell_{t,\istar}]^\kappa \geq \Delta \E_t[(\ell_{t,i}-\ell_{t,\istar})^2]$ for all $t$ and $i \neq \istar$, then with $m_t' = \ell_{t-1}$, $\reg(e_\istar) = \bigO{\rbr{\frac{\ln(dT)}{\Delta}}^{\frac{1}{2-\kappa}}T^{\frac{1-\kappa}{2-\kappa}}}$ holds both in expectation and with high probability.
\end{itemize}
\end{theorem}

The second condition in \pref{thm:fast_rates} is called the Bernstein condition and covers many interesting scenarios~\citep{koolen2016combining}.
Note that in this case with $m_t' = \ell_{t-1}$, the algorithm {\it simultaneously} ensures a path-length bound $\order\Big(\sqrt{\ln(dT)\sumt \|\ell_t - \ell_{t-1}\|_\infty}\Big)$ (since $r_{t,i}' \leq 2\|\ell_t - \ell_{t-1}\|_\infty$), which is useful for slowly changing environments such as some game playing settings~\citep{rakhlin2013optimization, syrgkanis2015fast}.
In \pref{sec:app_OLO}, we also give an application for OLO.

We close this subsection with the following two remarks.

\paragraph{Differences in algorithms.}
We note that most existing algorithms discussed above are variants of either \Prod~\citep{sani2014exploiting, gaillard2014second} or ``tilted exponential weight''~\citep{koolen2015second, wintenberger2017optimal},\footnote{The name ``tilted exponential weight'' is taken from~\citep{van2016metagrad}.}
which are somewhat similar to OMD with entropy regularizer.
However, even if some of them adopt individual time-varying learning rates as well,
they are different from our algorithm, as evidenced by the fact that these algorithm all take a closed ``proportional'' form, while our algorithm does not even when $\Omega_t=\simplex$ (see~\citep{bubeck2017online}).
We are also only able to obtain our guarantee with a general $m_t$ using this OMD framework but not the other methods (even though they achieve the bound for some special $m_t$ as discussed). We conjecture that there are some subtle but fundamental differences between these algorithms.

\paragraph{Indeed impossible for bandits.}
It is natural to ask if the similar impossible tuning is in fact also possible for the more challenging multi-armed bandit problem~\citep{auer2002nonstochastic}, where the minimax regret is $\bigo{\sqrt{dT}}$.
In other words, is it possible to achieve $\reg(e_i) = \tilO{\sqrt{d\sumt \ell_{t,i}^2}}$ for all $i$ in multi-armed bandits?
It turns out that this is {\it indeed impossible}, as a bound in this form would violate the multi-scale lower bound 
shown in~\citep[Theorem 23]{bubeck2017online}.

\subsection{A new master algorithm}\label{subsec:master}
Next, we instantiate \alg differently to obtain a master algorithm \master that combines a set of base algorithms and adaptively learns the best one (see \pref{alg:master}).
We will apply this master to both the expert problem (\pref{sec:app_expert}) and more generally the OLO problem (\pref{sec:app_OLO}) where the decision set generalizes from $\simplex$ to an arbitrary closed convex set $\calK$.

The instantiation still leaves the choices of $\Omega_t$ open for now and simply fixes the learning rate for each expert to be the same value over the $T$ rounds.
Since we will use this master, which itself deals with an expert problem with different base algorithms as experts, to deal with another expert/OLO problem, we adopt a different set of notations for the master.
Specifically, the set of expert is denoted by $\calE$, which consists of pairs in the form $(\eta, \calA)$ where $\eta$ is the learning rate for this expert and $\calA$ is a base algorithm.
For each expert $k = (\eta, \calA) \in \calE$, we use $\eta_k$ to denote the corresponding learning rate $\eta$.

\master maintains two sequences of distributions $p_1, \ldots, p_T$ and $p_1', \ldots, p_T'$ over the set of experts.
We use $\simplex[\calE]$ to denote the set of such distributions and $p_{t,k}$ to denote the weight assigned to expert $k$ by $p_t$.
We fix a specific initial distribution $p_1'$ such that $p_{1,k}' \propto \eta_k^2$.
Upon receiving the prediction $m_t \in \fR^d$ for the expert/OLO problem we are trying to solve, we feed it to all base algorithms, receive their decisions $\{w_t^k\}_{k\in\calE}$, and then define the prediction $h_t \in \fR^\calE$ for the master expert problem with $h_{t,k} = \inner{w_t^k}{m_t}$, that is, the predicted loss of the decision $w_t^k$.
Next, \master decides a subset $\Lambda_t \in \simplex[\calE]$ and performs the OMD update with the regularizer $\psi(p) = \sum_{k\in\calE} \frac{1}{\eta_{k}}p_k\ln p_k$ to compute $p_t$; note that the regularizer is now fixed over time.

With $p_t$, \master aggregates the decisions of all base algorithms by playing the convex combination $\sum_{k\in\calE} p_{t,k}w_t^k$.
After seeing the loss vector $\ell_t$ and feeding it to all base algorithms,
\master naturally defines the loss vector $g_t \in \fR^\calE$ for its own expert problem with $g_{t,k} = \inner{w_t^k}{\ell_t}$ and the corresponding correction term $b_t$ with $b_{t,k} = 32\eta_{k}(g_{t,k}-h_{t,k})^2$.
Finally, $p_{t+1}'$ is calculated according to the OMD update rule using $g_t + b_t$.

\setcounter{AlgoLine}{0}
\begin{algorithm}[t]
\caption{\master}
\label{alg:master}
	\textbf{Input:} a set of (learning rate, base algorithm) pairs $\calE$. 
	
	\textbf{Initialize:} $p_1' \in \simplex[\calE]$ such that $p_{1,k}' \propto \eta_k^2$ for each $k \in \calE$.
		
	\For{$t=1,\ldots,T$}{
	     Receive prediction $m_t \in \fR^d$ and feed it to all base algorithms.
	     
	     For each $k\in \calE$, receive decision $w_t^k \in \calK$ from the base algorithm and define $h_{t,k} = \inner{w_t^k}{m_t}$.
	    	     
	     Decide a compact convex decision subset $\Lambda_t \subseteq \simplex[\calE]$. 
	
	     Compute $p_t = \argmin_{p\in\Lambda_t}\inner{p}{h_t} + D_{\psi}(p, p_t')$ where $\psi(p) = \sum_{k\in\calE} \frac{1}{\eta_{k}}p_k\ln p_k$. 
	     
	     Play $w_t = \sum_{k\in\calE} p_{t,k}w_t^k \in \calK$, receive $\ell_t$ and feed it to all base algorithms.
	     
	     For each $k \in \calE$, define $g_{t,k} = \inner{w_t^k}{\ell_t}$ and $b_{t, k} = 32\eta_{k}(g_{t,k}-h_{t,k})^2$.
	     
		Compute $p_{t+1}' = \argmin_{p\in\Lambda_t}\inner{p}{g_t+b_t} + D_{\psi}(p, p_t')$. 
	}
\end{algorithm}

To use \master, one simply designs a set of base algorithms with corresponding learning rates (and decides the subset $\Lambda_t$ which is usually the set of distributions over some or all of the experts).
These base algorithms are usually different instances of the same algorithm with different parameters such as a different learning rate, which usually coincides with the learning rate $\eta_k$ for this expert.
The point of having this construction is that \master can then learn the best parameter setting of the base algorithm automatically.
Indeed, with $\reg_\calA$ being the regret of base algorithm $\calA$, we have the following guarantee that is a direct corollary of \pref{lem:framework}.

\begin{theorem}\label{thm:master}
Suppose that for all $t$, $32\eta_k |\inner{w_t^k}{\ell_t-m_t}| \leq 1$ holds for all $k\in \calE$ with $p_{t,k}>0$. 
Then for any $\kstar = (\etastar, \Astar) \in \calE$ such that $e_{\kstar} \in \bigcap_{t=1}^T \Lambda_t$, \master ensures 
\begin{equation}\label{eq:master_bound}
\forall u \in\calK, \;\;\reg(u) \leq \reg_{\Astar}(u) + \frac{1}{\etastar}\ln\rbr{\frac{\sum_k \eta_k^2}{\etastar^2}} + \frac{\sum_k \eta_k}{\sum_k \eta_k^2}
+32\etastar \sumt \inner{w_t^\kstar}{\ell_t - m_t}^2.
\end{equation}
\end{theorem}

The proof is deferred to \pref{app:framework}.
In all our applications, the learning rates are chosen from an exponential grid such that $\sum_k \eta_k$ and $\sum_k \eta_k^2$ are both constants.
Moreover, the term $32\etastar \sumt \inners{w_t^\kstar}{\ell_t - m_t}^2$ can usually be {\it canceled} by the a negative term from $\reg_{\Astar}(u)$, making the overhead of the master simply be $\bigo{\frac{1}{\etastar}\ln\frac{1}{\etastar}}$, which is rather small.
We remark that the idea of combining a set of base algorithms or more specifically ``learning the learning rate'' has appeared in many prior works such as~\citep{koolen2014learning, van2016metagrad, foster2017parameter, cutkosky2019combining, bhaskara2020online}.
However, the special regret guarantee of \alg that does not exist before allows us to derive new applications as shown in the next two sections.

%% file: app_expert.tex

In this section, we apply \master to derive yet another four new results for the expert problem (thus $\calK = \simplex$ throughout this section).
These results improve over the guarantee of \pref{thm:impossible_tuning} by respectively adapting to an arbitrary competitor and a prior, the scale of each expert, a switching sequence of competitors, and unknown loss ranges.\footnote{%
While we present all results using the master with appropriate base algorithms,
it is actually possible to ``flatten'' this two-layer structure to just one layer by duplicating each expert and assigning each copy a different learning rate.
We omit the details since this approach does not generalize to OLO.
}

\subsection{Adapting to an arbitrary competitor}\label{subsec:KL}
Typical regret bounds for the expert problem compete with an individual expert and pay for a $\sqrt{\ln d}$ factor.
Several works generalize this by replacing $\ln d$ with $\KL(u, \pi)$ when competing with an arbitrary competitor $u\in\simplex$, where $\pi$ is a fixed prior distribution over the experts~\citep{luo2015achieving, koolen2015second}.
Importantly, the bound holds simultaneously for all $u$. 
Inspired by these works, our goal here is to make the same generalization for 
\pref{thm:impossible_tuning}.
%
To do so, we again instantiate \alg differently to create a set of base algorithms, each with a fixed learning rate across all $i$ and $t$ (so both the master and the base algorithms are instances of \alg). 
Specifically, consider the following set of $\bigO{\ln T}$ experts:
\begin{equation}\label{eq:KL_base_alg}
\begin{split}
\KLbase &= \Big\{ (\eta_k, \calA_k): \forall k =1, \ldots, \lceil\log_2 T \rceil, \eta_k = \tfrac{1}{32\cdot 2^{k}}, \text{$\calA_k$ is \alg with $w_1' = \pi$,} \\
&\qquad\qquad \text{$\Omega_t = \simplex$, and $\eta_{t,i}=2\eta_k$ for all $t$ and $i$} \Big\}.
\end{split}
\end{equation}
By \pref{lem:framework}, we know that $\calA_k$ guarantees for all $u\in \simplex$:
\begin{equation}\label{eq:KL_base_alg_bound}
\reg_{\calA_k}(u) \leq \frac{\KL(u, \pi)}{2\eta_k} + 64\eta_k\sumt\sumi u_i(\ell_{t,i}-m_{t,i})^2 - 32\eta_k\sumt\sumi w_{t,i}^k(\ell_{t,i} - m_{t,i})^2.
\end{equation}
\master can then learn the best $\eta_k$ to achieve the optimal tuning.
Indeed, directly combining the guarantee of \master from \pref{thm:master} and noting that, importantly, the last term in \pref{eq:master_bound} can be canceled by the last negative term in \pref{eq:KL_base_alg_bound} by Cauchy-Schwarz inequality, we obtain the following result (full proof deferred to \pref{app:expert}).

\begin{theorem}\label{thm:KL}
Suppose $\|\ell_{t} - m_{t}\|_\infty \leq 1, \forall t$.
Then for any $\pi \in \simplex$, \master with expert set $\KLbase$ 
and $\Lambda_t = \simplex[\KLbase]$ ensures
$
\reg(u) = \bigO{\KL(u, \pi)+\ln V(u) + \sqrt{(\KL(u, \pi)+\ln V(u)) V(u)}}
$
for all $u\in\simplex$, where $V(u) = \max\Big\{3, \sumt\sumi u_i(\ell_{t,i}-m_{t,i})^2\Big\}$.
\end{theorem}

This result recovers the guarantee in \pref{thm:impossible_tuning} when $u = e_\istar$ and $\pi$ is uniform (in fact, it also improves the $\ln T$ factor to $\ln V(e_\istar)$).\footnote{However, we believe that the result of \pref{thm:impossible_tuning} is
still valuable since the algorithm does not require maintaining multiple base algorithms and is more computationally efficient and practical.}
Note that the implications discussed in \pref{subsec:impossible_tuning} by selecting different $m_t$ still apply here with the same improvement (from $\ln(dT)$ to $\KL(u, \pi)+\ln V(u)$).
In particular, this means that our results recover and improve those of~\citep{luo2015achieving, koolen2015second} (which only cover the case with $m_t = \inner{w_t}{\ell_t}\one$).

\subsection{Adapting to Multiple Scales}
Consider the ``multi-scale'' expert problem~\citep{bubeck2017online, foster2017parameter, cutkosky2018black} where each expert $i$ has a different loss range $c_i>0$ such that $|\ell_{t,i}|\leq c_i$ (and naturally $|m_{t,i}|\leq c_i$) for all $t$.
Previous works all achieve a bound $\reg(e_\istar)=\tilo{c_\istar\sqrt{T\ln d}}$, scaling only in terms of $c_\istar$.
The main term of our bound in \pref{thm:impossible_tuning} is already {\it strictly better} since the term $\sqrt{\sumt (\ell_{t,\istar}-m_{t,\istar})^2} \leq 2c_\istar\sqrt{T}$ inherently only scales with $c_\istar$.
The issue is that the lower-order term in the bound is in fact in terms of $\max_i c_i$.
To improve it to $c_\istar$, we apply similar ideas of \pref{subsec:KL} and again use \master to learn the best learning rate for the base algorithm \alg.
To this end, first define a set $\calS=\left\{k\in\fZ: \exists i\in [d], 
c_i \leq 2^{k-2} \leq c_i\sqrt{T}\right\}$ so that $\{\tfrac{1}{32\cdot 2^{k}}\}_{k\in \calS}$ contains all the learning rates we want to search over.
Then define expert set:
\begin{equation}\label{eq:ms_base_alg}
\begin{split}
\msbase &= \Big\{ (\eta_k, \calA_k): \forall k \in \calS, \eta_k = \tfrac{1}{32\cdot 2^{k}}, \text{$\calA_k$ is \alg with $w_1'$ being uniform over $\calZ(k)$,} \\
&\qquad\qquad \text{$\Omega_t = \cbr{w \in \Delta_d: w_i=0, \forall i \notin \calZ(k) }$, and $\eta_{t,i}=2\eta_k$ for all $t$ and $i$} \Big\},
\end{split}
\end{equation}
where 
$\calZ(k)=\{i\in [d]: c_i  \leq 2^{k-2}\}$.
Compared to \pref{eq:KL_base_alg}, another difference is that we restrict each base algorithm $\calA_k$ to work with only a subset $\calZ(k)$ of arms, which ensures the condition $32\eta_{t,i}|\ell_{t,i}-m_{t,i}| \leq 128\eta_k c_i \leq 1$ (for $i$ with $w_{t,i}^k > 0$)
of \pref{lem:framework} and similarly the condition of \pref{thm:master}.
With this construction, we can then automatically learn the best instance and achieve the following multi-scale bound that is a strict improvement of aforementioned previous works.
\begin{theorem}
	\label{thm:MS}
	Suppose for all $t$, $|\ell_{t,i}|\leq c_i$ and $|m_{t,i}|\leq c_i$ for some $c_i > 0$. Define $\cmin=\min_i c_i$ and $\Gamma_i=\ln(\frac{dTc_i}{\cmin})$.
	Then \master with expert set $\msbase$ defined in \pref{eq:ms_base_alg} and $\Lambda_t=\Delta_{\mss}$ ensures:
$
		\reg(e_{\istar}) = \otil\Big( c_{\istar}\Gamma_{\istar} + \sqrt{\Gamma_{\istar}\sumt(\ell_{t,\istar}-m_{t,\istar})^2}\Big)
$ for all $\istar \in [d]$.
\end{theorem}


\subsection{Adapting to a switching sequence}\label{subsec:switching}
So far, the regret measure we have considered compares with a fixed competitor across all $T$ rounds.
A more challenging notion of regret, called {\it switching regret}, compares with a sequence of changing competitors with a certain number of switches, which is a much more appropriate measure for non-stationary environments.
Specifically, we use $\calI$ to denote an interval of rounds (that is, a subset of $[T]$ in the form of $\{s, s+1, \ldots, t-1, t\}$) and $\reg^\calI(u) = \sum_{t\in\calI}\inner{w_t - u}{\ell_t}$ to denote the regret against $u$ on this interval.
For a partition $\calI_1, \ldots, \calI_S$ of $[T]$ and competitors $u_1, \ldots, u_S \in \simplex$, the corresponding switching regret is then $\sumj \reg^{\calI_j}(u_j)$.

Now, we show that almost the same construction as in \pref{subsec:KL} generalizes our result in \pref{thm:impossible_tuning} to switching regret as well.
Specifically, we deploy the following expert set:
\begin{equation}\label{eq:switch_base_alg}
\begin{split}
\switchbase &= \Big\{ (\eta_k, \calA_k): \forall k =1, \ldots, \lceil\log_2 T \rceil, \eta_k = \tfrac{1}{32\cdot 2^{k}}, \text{$\calA_k$ is \alg with $w_1' = \tfrac{1}{d}\one$,} \\
&\qquad\qquad \text{$\Omega_t = \cbr{w \in \Delta_d: w_i \geq \tfrac{1}{dT}}$, and $\eta_{t,i}=2\eta_k$ for all $t$ and $i$} \Big\},
\end{split}
\end{equation}
where the only essential difference compared to $\KLbase$ is the the use of a truncated simplex for $\Omega_t$.
We then have the following new switching regret guarantee.

\begin{theorem}\label{thm:switch}
If $\|\ell_{t} - m_{t}\|_\infty \leq 1$ holds for all $t \in [T]$,
then \master with expert set $\switchbase$ defined in \pref{eq:switch_base_alg} and $\Lambda_t = \cbr{p\in \simplex[\KLbase]: p_k \geq \frac{1}{T}}$ ensures for any partition $\calI_1, \ldots, \calI_S$ of $[T]$ and competitors $u_1, \ldots, u_S \in \simplex$,
\begin{equation}\label{eq:switch_bound}
\sumj \reg^{\calI_j}(u_j) = \bigO{S\ln(dT) + \sumj \sqrt{\ln(dT)\sum_{t\in\calI_j}\sumi u_{j,i}(\ell_{t,i}-m_{t,i})^2 }}.
\end{equation}
\end{theorem}

Our bound is never worse than the typical one $\bigo{\sqrt{ST\ln(dT)}}$ (due to Cauchy-Schwarz inequality) and significantly improves over previous works such as~\citep{cesa2012mirror, luo2015achieving} by again choosing different $m_t$ according to the discussions in \pref{subsec:impossible_tuning}. It also resolves an open problem raised by \cite{lu2019adaptive} on the possibility of making the switching regret bound adapt to the path length of the comparator sequence. 
The proof of \pref{thm:switch} requires a more general version of \pref{lem:framework} and is deferred to \pref{app:expert}.

\paragraph{Impossibility for interval regret.}
Looking at \pref{eq:switch_bound}, one might wonder whether the natural bound $\reg^{\calI_j}(u_j) = \order\big(\ln(dT) + \sqrt{\ln(dT)\sum_{t\in\calI_j}\sumi u_{j,i}(\ell_{t,i}-m_{t,i})^2}\big)$ holds for each interval $\calI_j$ separately.
Indeed, \pref{eq:switch_bound} could be derived from this (by summing over $S$ intervals).
It turns out that, even if its special case with $m_t = \inner{w_t}{\ell_t}$ is achievable~\citep{luo2015achieving}, this cannot hold in general as shown in \pref{app:impossible_interval}.
We find this intriguing (given that \pref{eq:switch_bound} is achievable) and reminiscent of the impossibility result for interval regret in bandits~\citep{daniely2015strongly}.

\subsection{Adapting to unknown loss ranges}\label{subsec:unknown_range}
The recent work of~\citep{mhammedi2019lipschitz} improves~\citep{koolen2015second} by adapting to the unknown loss range $\|\ell_t\|_\infty$.
Here, we show that \master is readily capable of dealing with such cases as well.
The high-level idea is to have each base algorithm to deal with a different possible loss range --- a larger loss range is handled by a smaller learning rate.
Once the loss becomes larger than what a base algorithm can handle, we remove this algorithm from the expert set, simply implemented by defining $\Lambda_t$ to be a subset of distributions that put zero weight on this base algorithm.
The removal of these base algorithms is necessary to ensure that the condition $32\eta_k  |\inner{w_t^k}{\ell_t-m_t}| \leq 1$ of \pref{thm:master} always holds.
We defer the details to \pref{app:unknown_range}, which include some additional techniques similar to those of~\citep{mhammedi2019lipschitz} such as feeding the algorithm with truncated fake losses and a restarting scheme.
Our final result is summarized below.

\begin{theorem}\label{thm:unknown_range}
Let $\max_t \|\ell_t - m_t\|_\infty$ be unknown.
For any prior $\pi \in \simplex$, 
\pref{alg:unknown_range} (with input \pref{eq:expert_unknown_range} and $B_0$) ensures
$
\reg(u) = \bigO{B(\KL(u, \pi)+\ln T) + \sqrt{(\KL(u, \pi)+\ln T) V(u)}}, \;\forall u\in\simplex
$,
where $V(u) = \max\cbr{3, \sumt\sumi u_i(\ell_{t,i}-m_{t,i})^2}$ and $B=\max\{B_0, \max_t \|\ell_t - m_t\|_\infty\}$. 
\end{theorem}

Note that $B$ is in terms of the maximum range of the predicted error as opposed to $\max_t \|\ell_t\|_\infty$ used in~\citep{mhammedi2019lipschitz},
and could be much smaller when the prediction is accurate.
Besides, \citep{mhammedi2019lipschitz} only achieves the bound with $m_t = \inner{w_t}{\ell_t}\one$ in $V(u)$. 

%% file: app_OLO.tex

We next discuss applications of \master to general OLO.
For simplicity, we assume that $\calK$ is a compact convex set such that $\norm{w} \leq D$ for all $w\in\calK$, 
and also $\max_t \norm{\ell_t-m_t} \leq 1$, 
where $\norm{\cdot}$ is $L_2$ norm (extensions to general primal-dual norm are straightforward).
In \pref{app:unconstrained}, we show that all our results can be generalized to the unconstrained setting where $\calK$ is unbounded and also the unknown Lipschitzness setting where $\max_t \norm{\ell_t-m_t}$ is unknown ahead of time.


\paragraph{Application 1: combining Online Newton Step}
It is a folklore that one can reduce OLO to the expert problem by discretizing the decision set $\calK$ into $\bigO{T^d}$ points and treating each point as an expert.
With this reduction, our result in \pref{thm:impossible_tuning} immediately implies a bound $\reg(u) = \tilo{\sqrt{d\sum_t \inner{u}{\ell_t - m_t}^2}}$ for OLO.
Of course, the caveat is that the reduction is computationally inefficient.\footnote{%
The reduction is efficient when $d=1$ though.
This gives an alternative algorithm with the same guarantee as~\citep[Theorem 1]{cutkosky2018black} and is useful already with their reduction from general $d$ to $d=1$.
}
Below, we show that the same (or even better) bound can be achieved efficiently by using \master with a variant of Online Newton Step (ONS)~\citep{hazan2007logarithmic} as the base algorithm.
Specifically, the ONS variant (denoted by $\calA_k$ and parameterized by a fixed learning rate $\eta$) can be presented in the OMD framework again using an auxiliary cost function $c_t(w) = \inner{w}{\ell_t}+32\eta\inner{w}{\ell_t - m_t}^2$ and a time-varying regularizer $\psi_t(w) = \frac{1}{2}\norm{w}_{A_t}^2$ where $A_t = \eta\rbr{2I+\sum_{s<t}(\nabla_s-m_s)(\nabla_s-m_s)^\top}$ and $\nabla_s = \nabla c_s(w_s^k)$.
This variant is similar to that in~\citep{cutkosky2018black}, but incorporates the prediction $m_t$ as well.
We defer the details to \pref{app:ONS}, which shows: $\calA_k$ ensures
(with $r$ being the rank of $\calL_T =\sumt(\ell_t-m_t)(\ell_t-m_t)^\top$)
\begin{equation}\label{eq:ONS_bound}
\reg(u) \leq \tilO{\frac{r}{\eta}+\eta \sumt\inner{u}{\ell_t-m_t}^2} - 16\eta\sumt\inner{w_t^k}{\ell_t -m_t}^2.
\end{equation}
Therefore, using \master to learn the best learning rate and noting that the last negative term in \pref{eq:ONS_bound} cancels the last term in \pref{eq:master_bound}, we obtain the following result.
\begin{theorem}\label{thm:ONS}
Let $r \leq d$ be the rank of $\calL_T =\sumt(\ell_t-m_t)(\ell_t-m_t)^\top$.
\master with expert set $\ONSbase$ defined in \pref{eq:ONS_base_alg} and $\Lambda_t = \simplex[\ONSbase]$ ensures
\begin{equation}\label{eq:ONS_master_bound}
\forall u \in \calK, \;\;  \reg(u) = \tilO{r\norm{u}+ \sqrt{r\sumt\inner{u}{\ell_t-m_t}^2}}. 
\end{equation}
\end{theorem}

Similar bounds appear before but only with $m_t=0$~\citep{cutkosky2018black, cutkosky2020better}, and we are not able to incorporate general $m_t$ into their algorithms.
Our bound has no explicit dependence on $D$ at all, and its dependence on $\ell_t - m_t$ is only through its projection on $u$. 

\paragraph{Application 2: combining Gradient Descent}
Another natural choice of base algorithm is Optimistic Gradient Descent, which guarantees $\reg(u) = \order\big(\frac{\norm{u}^2}{\eta}+\eta \sumt\norm{\ell_t-m_t}^2\big)$ (see \pref{app:GD}).
Combining instances with different learning rates that operate over subsets of $\calK$ of different sizes (necessary to ensure $32\eta_k  |\inner{w_t^k}{\ell_t-m_t}| \leq 1$ for \pref{thm:master}), we obtain:

\begin{theorem}\label{thm:GD}
\master with expert set $\GDbase$ defined in \pref{eq:GD_base_alg} and $\Lambda_t = \simplex[\GDbase]$ ensures
\begin{equation}\label{eq:GD_master_bound}
\forall u \in \calK, \;\; \reg(u) = \tilO{\norm{u}+\norm{u}\sqrt{\sumt\norm{\ell_t-m_t}^2}}. 
\end{equation}
\end{theorem}

This bound appears before first with $m_t=0$ in~\citep{cutkosky2018black} and later with general $m_t$ in~\citep{cutkosky2019combining}.
We recover the bound easily with our framework.
Similar to \pref{eq:ONS_master_bound}, this bound adapts to the size of the competitor $u$ (with no dependence on $D$).
An advantage of \pref{eq:GD_master_bound} is that it is dimension-free, while \pref{eq:ONS_master_bound} is potentially large for high-dimensional data.

\paragraph{Application 3: combining AdaGrad}
Inspired by the recent work of~\citep{cutkosky2020better} that provides an improved guarantee of the full-matrix version of AdaGrad~\citep{duchi2011adaptive},
we next design an optimistic version of AdaGrad and combine instances with different parameters to obtain the following new result.

\begin{theorem}\label{thm:AdaGrad}
\master with expert set $\AdaGradbase$ defined in \pref{eq:AdaGrad_base_alg} and $\Lambda_t = \simplex[\AdaGradbase]$ ensures
\begin{equation}\label{eq:AdaGrad_master_bound}
\forall u \in \calK, \;\; \reg(u) = \tilO{\norm{u} + \sqrt{\rbr{u^\top(I+\calL_T)^{1/2}u }\trace{\calL_T^{1/2}}}}.
\end{equation}
\end{theorem}

All details are deferred to \pref{app:AdaGrad}.
\citet{cutkosky2020better} achieves \pref{eq:AdaGrad_master_bound} for $m_t=0$ (again, we are not able to extend their algorithm to deal with general $m_t$).
The three types of bounds we have shown in \pref{eq:ONS_master_bound}, \pref{eq:GD_master_bound}, and \pref{eq:AdaGrad_master_bound} are {\it incomparable}, that is, there are cases for each one to be the smallest;
see \citep{cutkosky2020better} for in-depth discussions with $m_t=0$.
However, since the configuration of \master is the same in all these three results (other than the expert set), we can in fact achieve {\it the best of three worlds} by feeding the union of these three expert set to \master, summarized in the following corollary.

\begin{corollary}[Best-of-three-worlds]\label{cor:combining}
\master with expert set $\calE = \ONSbase \cup \GDbase \cup \AdaGradbase$ and $\Lambda_t = \simplex[\calE]$ ensures regret bounds \pref{eq:ONS_master_bound}, \pref{eq:GD_master_bound}, and \pref{eq:AdaGrad_master_bound} simultaneously.
\end{corollary}

We remark that the technique proposed in \citep{cutkosky2019combining} can similarly combine algorithm's guarantees with little overhead, but it only works for the unconstrained setting.
It is tempting to apply the unconstrained-to-constrained reduction from~\citep{cutkosky2018black} to lift this restriction, but that does not work generally as discussed in \citep[Section 4]{cutkosky2020better}.
All in all, we are not aware of any other methods capable of achieving this best-of-three-worlds result.

\paragraph{Application 4: combining MetaGrad's base algorithm}
Finally, we discuss how to recover and generalize the regret bound of MetaGrad~\citep{van2016metagrad} which depends on the sum of squared instantaneous regret and is the analogue of the \MLProd guarantee for the expert problem.
Our base algorithm is yet another variant of ONS that uses a different auxiliary cost function $c_t(w) = \inner{w}{\ell_t}+32\eta\inner{w-w_t}{\ell_t - m_t}^2$ with an extra offset in terms of $w_t$ (the decision of the master).
When $m_t = 0$ this is the same base algorithm used in~\citep{van2016metagrad}.
Compared to \pref{eq:ONS_bound}, this variant ensures the following
\begin{equation}\label{eq:MG_bound}
\reg(u) \leq \tilO{\frac{r}{\eta}+\eta \sumt\inner{u-w_t}{\ell_t-m_t}^2} - 16\eta\sumt\inner{w_t^k-w_t}{\ell_t -m_t}^2.
\end{equation}
Note that the last negative term is now slightly different from the last term in \pref{eq:master_bound}.
To make them match, we need to change the definition of $h_{t,k}$ in \master from $h_{t,k}' \defeq \inner{w_t^k}{m_t}$ to $h_{t,k}' + \inner{p_t}{g_t-h_t'}$, the same trick used in Implication 4 of \pref{subsec:impossible_tuning}
(this is also the reason why we cannot include this result in \pref{cor:combining} as well).
We defer the details to \pref{app:MetaGrad} and show the final bound below.

\begin{theorem}\label{thm:MetaGrad}
Let $r \leq d$ be the rank of $\sumt(\ell_t-m_t)(\ell_t-m_t)^\top$.
\master with the new definition of $h_t$ described above, expert set $\MetaGradbase$ defined in \pref{eq:MetaGrad_base_alg}, and $\Lambda_t = \simplex[\MetaGradbase]$ ensures
$\forall u \in \calK, \reg(u) = \Big(rD + \sqrt{r\sumt\inner{u-w_t}{\ell_t-m_t}^2}\Big)$. 
\end{theorem}

This bound generalizes the MetaGrad's guarantee from $m_t=0$ to general $m_t$ and is the analogue of the bound discussed in Implication 4 of \pref{subsec:impossible_tuning} for the expert problem.
Similarly to~\pref{thm:fast_rates}, when using $m_t=\ell_{t-1}$, our bound preserves all the fast rate consequences discussed in~\citep{van2016metagrad, koolen2016combining}, 
while ensuring a bound in terms of only the variation of the loss vectors $\sum_t \norm{\ell_t-\ell_{t-1}}^2$.
We remark that MetaGrad also uses a master algorithm to combine similar ONS variants, but the master is ``tilted exponential weight'' and cannot incorporate general $m_t$.

%% file: appendix_lemma.tex

\begin{lemma}\label{lem:omd}
	Define $w^\star=\argmin_{w\in\calK}\inner{w}{x} + D_{\psi}(w, w')$ for some compact convex set $\calK \subset \fR^d$, convex function $\psi$, an arbitrary point $x \in \fR^d$, and a point $w'\in \calK$. Then for any $u\in\calK$:
	\begin{align*}
		\inner{w^{\star}-u}{x} \leq D_{\psi}(u, w') - D_{\psi}(u, w^{\star}) - D_{\psi}(w^{\star}, w').
	\end{align*}
\end{lemma}

\begin{proof}
This is shown for example in the proof of~\citep[Lemma~1]{wei2018more}, and is by direct calculations plus the first-order optimality condition of $w^\star$.
\end{proof}

\begin{lemma}
	\label{lem:oomd} 
	Let $w_t=\argmin_{w\in\calK}\inner{w}{m_t} + D_{\psi_t}(w, w'_t)$ and $w'_{t+1}=\argmin_{w\in\calK}\inner{w}{\ell_t} + D_{\psi_t}(w, w'_t)$
	 for some compact convex set $\calK \subset \fR^d$, convex function $\psi_t$,  arbitrary points $\ell_t, m_t \in \fR^d$, and a point $w_t'\in \calK$. 
	 Then, for any $u\in\calK$ we have
	\begin{align*}
		\inner{w_t - u}{\ell_t} \leq \inner{w_t - w'_{t+1}}{\ell_t - m_t} + D_{\psi_t}(u, w'_t) - D_{\psi_t}(u, w'_{t+1}) - D_{\psi_t}(w'_{t+1}, w_t) - D_{\psi_t}(w_t, w'_t). 
	\end{align*}
\end{lemma}
\begin{proof}
	We apply \pref{lem:omd} with $w^\star=w_t, u=w'_{t+1}$ to obtain
	\begin{align*}
		\inner{w_t - w'_{t+1}}{m_t} &\leq D_{\psi_t}(w'_{t+1}, w'_t) - D_{\psi_t}(w'_{t+1}, w_t) - D_{\psi_t}(w_t, w'_t), 
	\end{align*}
	and then with $w^\star=w'_{t+1}$ to obtain:
	\begin{align*}
		\inner{w'_{t+1}-u}{\ell_t} &\leq D_{\psi_t}(u, w'_t) - D_{\psi_t}(u, w'_{t+1}) - D_{\psi_t}(w'_{t+1}, w'_t).
	\end{align*}
	Summing the two inequalities above, we have:
	\begin{align*}
		\inner{w_t-w'_{t+1}}{m_t} + \inner{w'_{t+1}-u}{\ell_t} 
		&\leq D_{\psi_t}(u, w'_t) - D_{\psi_t}(u, w'_{t+1}) - D_{\psi_t}(w'_{t+1}, w_t) - D_{\psi_t}(w_t, w'_t). 
	\end{align*}
	Also note that the left-hand side is equal to:
		\begin{align*}
		\inner{w_t-w'_{t+1}}{m_t} + \inner{w'_{t+1}-u}{\ell_t} &= \inner{w_t-w'_{t+1}}{m_t-\ell_t} + \inner{w_t-w'_{t+1}}{\ell_t} + \inner{w'_{t+1}-u}{\ell_t}\\
		&= \inner{w_t-w'_{t+1}}{m_t-\ell_t} + \inner{w_t-u}{\ell_t}.
	\end{align*}
	Combining and reorganizing terms, we get the desired result.
\end{proof}

\begin{lemma}
	\label{lem:stability}
	For any convex function $\psi$ defined on convex set $\calK \subset \fR^d$ and a point $x\in\fR^d$, define $F_x(w)=\inner{w}{x} + \psi(w)$
	and $w_x = \argmin_{w\in\calK} F_x(w)$.
	Suppose that for some $x, x' \in \fR^d$, there is a constant $c$ such that 
	for all $\xi$ on the segment connecting $w_x$ and $w_{x'}$, 
	$\nabla^2\psi(\xi)\mgeq c\nabla^2\psi(w_{x})$ holds (which means $\nabla^2\psi(\xi) - c\nabla^2\psi(w_{x})$ is positive semi-definite).
	Then, we have $\inner{w_x-w_{x'}}{x'-x}\geq 0$ and $\norm{w_x-w_{x'}}_{\nabla^{2}\psi(w_x)} \leq \frac{2}{c}\norm{x-x'}_{\nabla^{-2}\psi(w_x)}$.
\end{lemma}
\begin{proof}
	Note that
	\begin{align*}
		&F_{x'}(w_x) - F_{x'}(w_{x'}) \\
		&= \inner{w_x-w_{x'}}{x'-x} + F_x(w_x) - F_x(w_{x'}) \tag{definition of $F$} \\
		&\leq \inner{w_x-w_{x'}}{x'-x} \tag{optimality of $w_x$} \\
		&\leq \norm{w_x-w_{x'}}_{\nabla^2\psi(w_x)}\norm{x'-x}_{\nabla^{-2}\psi(w_{x})} \tag{H\"{o}lder's inequality}.
	\end{align*}
	Using Taylor expansion, for some $\xi$ on the segment connecting $w_x$ and $w_{x'}$, we have
	\begin{align*}
		F_{x'}(w_x) - F_{x'}(w_{x'}) &= \inner{w_x-w'_x}{\nabla F_{x'}(w_{x'})} + \frac{1}{2}\norm{w_x-w_{x'}}_{\nabla^2\psi(\xi)}^2\\
		&\geq \frac{1}{2}\norm{w_x-w_{x'}}_{\nabla^2\psi(\xi)}^2 \tag{first-order optimality of $w_{x'}$} \\
		&\geq \frac{c}{2}\norm{w_x-w_{x'}}^2_{\nabla^2\psi(w_x)}. \tag{condition of the lemma}
	\end{align*}
	Combining we have, $\inner{w_x-w_{x'}}{x'-x}\geq F_{x'}(w_x)-F_{x'}(w_{x'})\geq c\norm{w_x-w_{x'}}^2_{\nabla^2\psi(w_x)}\geq 0$, and also
$\frac{c}{2}\norm{w_x-w_{x'}}^2_{\nabla^2\psi(w_x)}\leq \norm{w_x-w_{x'}}_{\nabla^2\psi(w_x)}\norm{x'-x}_{\nabla^{-2}\psi(w_{x})}$, which implies 
	\[
		\norm{w_x-w_{x'}}_{\nabla^2\psi(w_x)}\leq \frac{2}{c}\norm{x-x'}_{\nabla^{-2}\psi(w_x)}
	\]
	and finishes the proof.
\end{proof}

\begin{lemma}[Multiplicative Stability]\label{lem: stability for truncated one}
    Let $\Omega=\left\{w\in\Delta_d:~ w_i\geq b_i, \; \forall i \in[d]\right\}$ for some $b_i\in[0,1]$, $w' \in \Omega$ be such that $w'_i > 0$ for all $i \in [d]$, $ w = \argmin_{w\in\Omega}\left\{\langle w, \ell\rangle + D_{\psi}(w, w') \right\} $ where $\psi(w)=\sum_{i=1}^d \frac{1}{\eta_i} w_i \ln w_i$, $|\ell_i|\leq \cmax$, and $\eta_i \cmax\leq \frac{1}{32}$ for all $i$ and some $\cmax >0$.  Then $w_i\in [\frac{1}{\sqrt{2}}w_i', \sqrt{2}w_i']$. 
\end{lemma}

\begin{proof}
    Recall that $D_{\psi}(w,w')=\sum_i \frac{1}{\eta_i} \left(w_i \ln \frac{w_i}{w_i'} - w_i + w_i'\right)$. By the KKT condition of the optimization problem, we have for some $\lambda$ and $\mu_i\geq 0$,
    \begin{align*}
        \ell_i + \frac{1}{\eta_i}\ln\frac{w_i}{w_i'} - \lambda - \mu_i=0
    \end{align*}
    and $\mu_i\left(w_i-b_i\right)=0$ for all $i$. The above gives $w_i=w_i'\exp\left(\eta_i\left(-\ell_i+\lambda + \mu_i\right)\right)$.  
    We now separately discuss two cases.
    
    \paragraph{Case 1: $\min_i (\ell_i-\mu_i) \neq \max_i (\ell_i-\mu_i)$. } In this case, we claim that $\min_i (\ell_i-\mu_i) < \lambda < \max_i (\ell_i-\mu_i)$. We prove it by contradiction: 
    If $\lambda \geq \max_i (\ell_i-\mu_i)$, then 
 \[ \sum_{i}w_i = \sum_i w_i'\exp\left(\eta_i\left(-\ell_i+\lambda + \mu_i\right)\right) > \sum_i w_i' = 1 \]
    contradicting with $w\in\Delta_d$ (the strict inequality is because there exists some $j$ such that $\max_{i}(\ell_i-\mu_i) > (\ell_j - \mu_j)$ and $w_j' > 0$). We can derive a similar contradiction if $\lambda \leq \min_i (\ell_i-\mu_i)$. 
    Thus, we conclude $\min_i (\ell_i-\mu_i) < \lambda < \max_i (\ell_i-\mu_i)$.
    
    Our second claim is that for all $i$ with $\mu_i\neq 0$, $\ell_i-\mu_i\geq \lambda$. 
      Indeed, when $\mu_i\neq 0$, we have $b_i=w_i=w_i'\exp\left(\eta_i\left(-\ell_i+\lambda + \mu_i\right)\right)$. Clearly, $\exp\left(\eta_i\left(-\ell_i+\lambda + \mu_i\right)\right)\leq 1$ must hold; otherwise we have $w_i'<b_i$ which is a contradiction with $w' \in \Omega$. Therefore, $-\ell_i+\lambda+\mu_i\leq 0$.

Combining the above two claims, we see that $\min_i(\ell_i-\mu_i)$ must be equal to $\min_i\ell_i$; otherwise, we have $\min_i(\ell_i-\mu_i) < \min_i \ell_i$, which implies that there exists an $j$ such that $\min_i(\ell_i-\mu_i)=\ell_j-\mu_j$ and $\mu_j > 0$. By the first claim, $\lambda > \ell_j - \mu_j$, and this contradicts with the second claim. 

     Thus, $\max_i(\ell_i - \mu_i) - \min_i(\ell_i - \mu_i)=\max_i(\ell_i - \mu_i) - \min_i \ell_i \leq \max_i \ell_i - \min_i \ell_i \leq 2\cmax$ (the inequality is by $\mu_i\geq 0$).
     Since both $\lambda$ and $\ell_i - \mu_i$ are in the range $[\min_i(\ell_i - \mu_i), \max_i(\ell_i - \mu_i)]$, we have 
   $|-\ell_i + \lambda +\mu_i| \leq \max_i(\ell_i - \mu_i) - \min_i(\ell_i - \mu_i) \leq 2\cmax$. By the condition on $\eta_i$, we then have $w_i \in \left[\exp(-\frac{1}{16})w_i', \exp(\frac{1}{16})w_i'\right]\subset \left[\frac{1}{\sqrt{2}}w_i', \sqrt{2}w_i'\right]$. 

    \paragraph{Case 2: $\min_i (\ell_i-\mu_i) = \max_i (\ell_i-\mu_i)$. } In this case, it is clear that $\lambda=\ell_i-\mu_i$ must hold for all $i$ to make $w$ and $w'$ both distributions. Thus, $w_{t,i}=w_{t,i}'$ for all $i$. 
\end{proof}

%% file: appendix_framework.tex

In this section, we provide the omitted proofs for \pref{sec:framework}.	

\subsection{\pfref{lem:framework}}
\begin{proof}
     By \pref{lem:oomd}, we have (dropping one non-positive term)
     \begin{align}
		&\sumt\inner{w_t-u}{\ell_t+a_t} \nonumber \\
		&\leq \sumt \left(D_{\psi_t}(u, w'_t) - D_{\psi_t}(u, w'_{t+1})\right) + \sumt \rbr{\inner{w_t - w'_{t+1}}{\ell_t - m_t + a_t} - D_{\psi_t}(w'_{t+1}, w_t)}. \label{eq: regret decomp}
	\end{align}
    For the first term, we reorder it and use $D_{\psi_t}(u, v) = \sum_{i=1}^d\frac{1}{\eta_{t,i}}\kl(u_i, v_i)$: 
	\begin{align*}
		&\sumt \left(D_{\psi_t}(u, w'_t) - D_{\psi_t}(u, w'_{t+1})\right) 
		= D_{\psi_1}(u, w'_1) + \sum_{t=2}^{T} \left(D_{\psi_t}(u, w'_t) - D_{\psi_{t-1}}(u, w'_t)\right) - D_{\psi_T}(u, w_{T+1}') \\
		&\leq \sumi\frac{1}{\eta_{1,i}}\kl(u_i, w'_{1,i}) + \sum_{t=2}^T\sumi\rbr{\frac{1}{\eta_{t,i}} - \frac{1}{\eta_{t-1,i}} }\kl(u_i, w'_{t,i}). 
	\end{align*}
	
	For the second term, fix a particular $t$ and define $\wstar=\argmax_{w \in \fR^d_+}\inner{w_t-w}{\ell_t-m_t+a_t} - D_{\psi_t}(w, w_t)$.
	By the optimality of $\wstar$, we have: $\ell_t-m_t+a_t=\nabla\psi_t(w_t)-\nabla\psi_t(\wstar)$ and thus $\wstar_i=w_{t,i}e^{-\eta_{t,i}(\ell_{t,i}-m_{t,i}+a_{t,i})}$.
	Therefore, we have
	\begin{align*}
		&\inner{w_t - w'_{t+1}}{\ell_t - m_t + a_t} - D_{\psi_t}(w'_{t+1}, w_t)\\
		&\leq  \inner{w_t-\wstar}{\ell_t-m_t+a_t} - D_{\psi_t}(\wstar, w_t)\\
		&= \inner{w_t-\wstar}{\nabla\psi_t(w_t)-\nabla\psi_t(\wstar)} - D_{\psi_t}(\wstar, w_t)\\
		&= D_{\psi_t}(w_t, \wstar) = \sumi\frac{1}{\eta_{t,i}}\rbr{ w_{t,i}\ln\frac{w_{t,i}}{\wstar_i} - w_{t,i} + \wstar_i }\\
		&= \sumi \frac{w_{t,i}}{\eta_{t,i}}\rbr{ \eta_{t,i}(\ell_{t,i}-m_{t,i}+a_{t,i}) - 1 + e^{-\eta_{t,i}(\ell_{t,i}-m_{t,i}+a_{t,i})} }\\
		&\leq \sumi \eta_{t,i}w_{t,i}\rbr{\ell_{t,i}-m_{t,i}+a_{t,i}}^2,
	\end{align*}
	where in the last inequality we apply $e^{-x} - 1 + x \leq x^2$ for $x\geq -1$ and the condition of the lemma $\eta_{t,i}|\ell_{t,i}-m_{t,i}|\leq \frac{1}{32}$ such that $\eta_{t,i}|\ell_{t,i}-m_{t,i}+a_{t,i}|\leq \eta_{t,i}|\ell_{t,i}-m_{t,i}| + 32\eta_{t,i}^2(\ell_{t,i}-m_{t,i})^2 \leq \frac{1}{32} + \frac{32}{32^2}\leq \frac{1}{16}$.
	Using the definition of $a_t$ and the condition $\eta_{t,i}|\ell_{t,i}-m_{t,i}|\leq \frac{1}{32}$ again, we also continue with
	\begin{align*}
	\inner{w_t - w'_{t+1}}{\ell_t - m_t + a_t} - D_{\psi_t}(w'_{t+1}, w_t)
	&\leq \sumi \eta_{t,i}w_{t,i}\rbr{\ell_{t,i}-m_{t,i}+32\eta_{t,i}\rbr{\ell_{t,i}-m_{t,i}}^2}^2 \\
	&\leq 4\sumi \eta_{t,i}w_{t,i}\rbr{\ell_{t,i}-m_{t,i}}^2.
	\end{align*}

	To sum up, 
	combining everything, we have,
	\begin{align*}
	     &\sumt\inner{w_t-u}{\ell_t+a_t}\\
	     &\leq \sumi\frac{1}{\eta_{1,i}}\kl(u_i, w'_{1,i}) + \sum_{t=2}^T\sumi\rbr{\frac{1}{\eta_{t,i}} - \frac{1}{\eta_{t-1,i}} }\kl(u_i, w'_{t,i}) +  4\sumt\sumi \eta_{t,i}w_{t,i}(\ell_{t,i}-m_{t,i})^2. 
	\end{align*}
	Finally, moving $\sum_{t=1}^T \inner{w_t-u}{a_t}$ to the right-hand side of the inequality and using the definition of $a_t$ again finishes the proof. 
\end{proof}

\subsection{\pfref{thm:impossible_tuning}}
\begin{proof}
To apply \pref{lem:framework}, we notice that
the condition $32\eta_{t,i}|\ell_{t,i}-m_{t,i}|\leq 1$ of \pref{lem:framework} holds trivially by the definition of $\eta_{t,i}$.
Therefore, applying \pref{eq:general_bound} with $u = (1-\frac{1}{T})e_\istar + \frac{1}{T}w_1' \in \bigcap_{t=1}^T \Omega_t$, 
 we have:
	\begin{align}
	\reg(e_{\istar}) &= \reg(u) + \sumt\inner{u-e_{\istar}}{\ell_t} \notag\\
	&= \reg(u) + \frac{1}{T}\sumt\inner{w'_1-e_{\istar}}{\ell_t}  \notag\\
	&\leq \reg(u) + 2 \notag\\
	&\leq \sumi \frac{1}{\eta_{1,i}} \kl(u_i, w_{1,i}') + \sum_{t=2}^T \sumi \rbr{\frac{1}{\eta_{t,i}}-\frac{1}{\eta_{t-1,i}}}\kl(u_i, w_{t,i}') \notag\\
		&\qquad\qquad + 32 \sum_{t=1}^T\sumi \eta_{t,i} u_i (\ell_{t,i} - m_{t,i})^2 - 16  \sum_{t=1}^T\sumi \eta_{t,i}w_{t,i}(\ell_{t,i} - m_{t,i})^2 + 2. \label{eq:main_bound_of_analysis}
	\end{align}
	For the first term, note that $u_i\leq w'_{1,i}$ when $i\neq \istar$, and $\eta_{1,i}=\frac{1}{64}$. Thus,
	\begin{align*}
		\sumi\frac{1}{\eta_{1,i}}\kl(u_i, w'_{1,i}) &= \sumi \frac{1}{\eta_{1,i}}\rbr{ u_i\ln\frac{u_i}{w'_{1,i}} - u_i + w'_{1,i} } \leq 64 u_{\istar}\ln\frac{u_{\istar}}{w'_{1,\istar}} + \sumi 64 \cdot \frac{1}{d}=\order(\ln d). \\
	\end{align*}
	For the second term, we proceed as 
	\begin{align*}
		&\sum_{t=2}^T \sumi \rbr{\frac{1}{\eta_{t,i}}-\frac{1}{\eta_{t-1,i}}}\kl(u_i, w_{t,i}')\\
		&= \sum_{t=2}^T\sumi\rbr{\frac{1}{\eta_{t,i}}-\frac{1}{\eta_{t-1,i}}}\rbr{ u_{i}\ln\frac{u_{i}}{w'_{t,i}} - u_i + w_{t,i}'} \\
		&\leq \sum_{t=2}^T\rbr{\frac{1}{\eta_{t,\istar}}-\frac{1}{\eta_{t-1,\istar}}}\rbr{ u_{\istar}\ln\frac{u_{\istar}}{w'_{t,\istar}} } + \sum_{t=2}^T \sumi \rbr{\frac{1}{\eta_{t,i}}-\frac{1}{\eta_{t-1,i}}} w'_{t,i}  \tag{$u_i=\frac{1}{dT}\leq w_{t,i}'$ for $i\neq \istar$} \\
		&= \sum_{t=2}^T\rbr{\frac{1}{\eta_{t,\istar}}-\frac{1}{\eta_{t-1,\istar}}}\rbr{ u_{\istar}\ln\frac{u_{\istar}}{w'_{t,\istar}} } + \sum_{t=2}^T \sumi \frac{\frac{1}{\eta_{t,i}^2}-\frac{1}{\eta_{t-1,i}^2}}{\frac{1}{\eta_{t,i}}+\frac{1}{\eta_{t-1,i}}}w'_{t,i}  \\
		&\leq \frac{\ln(dT)}{\eta_{T,\istar}} + \sum_{t=2}^T\sumi\eta_{t-1,i}\rbr{ \frac{1}{\eta_{t,i}^2} - \frac{1}{\eta_{t-1,i}^2} }w'_{t,i} \tag{$u_{\istar}\ln\frac{u_{\istar}}{w'_{t,\istar}} \leq \ln(dT)$} \\
		&\leq 64\ln(dT)+ \sqrt{\ln (dT)\sumt(\ell_{t,\istar}-m_{t,\istar})^2} + \sum_{t=2}^T\sumi \frac{1}{\ln (dT)} \eta_{t-1,i}w'_{t,i}(\ell_{t-1, i} - m_{t-1, i})^2 \tag{by the definition of $\eta_{t,i}$}\\
		&\leq 64\ln(dT) + \sqrt{\ln (dT)\sumt(\ell_{t,\istar}-m_{t,\istar})^2} + \sum_{t=2}^T\sumi 2 \eta_{t-1,i}w_{t-1,i}(\ell_{t-1, i} - m_{t-1, i})^2, 
	\end{align*}
	where the last step uses the fact $w_{t,i}' \leq 2w_{t-1,i}$ according to the multiplicative stability lemma \pref{lem: stability for truncated one} (which asserts $w_{t,i}'\in [\frac{1}{\sqrt{2}}w_{t-1,i}', \sqrt{2}w_{t-1,i}']$ and $w_{t-1,i} \in [\frac{1}{\sqrt{2}}w_{t-1,i}', \sqrt{2}w_{t-1,i}']$).
	Note that the last term is then canceled by the fourth term of \pref{eq:main_bound_of_analysis}.
	For the third term of \pref{eq:main_bound_of_analysis}, we have
	\begin{align*}
		&\sum_{t=1}^T\sumi \eta_{t,i} u_i (\ell_{t,i} - m_{t,i})^2\\
		&\leq \sumt \eta_{t,\istar}(\ell_{t,\istar}-m_{t,\istar})^2 + \frac{1}{dT}\sumt \sumi \eta_{t,i}(\ell_{t,i}-m_{t,i})^2 \\
		&\leq \sumt \sqrt{\frac{\ln(dT)}{\sum_{s<t} (\ell_{s,\istar}-m_{s,\istar})^2}} \cdot (\ell_{t,\istar} - m_{t,\istar})^2 + 1 \\
		&\leq \order\left(\sqrt{\ln(dT)\sumt(\ell_{t, \istar} - m_{t, \istar})^2} + 1\right).
	\end{align*}
	Combining everything then proves
	\begin{align*}
		\reg(e_\istar) = \bigO{\ln(dT) + \sqrt{\ln(dT)\sumt(\ell_{t,\istar} - m_{t,\istar})^2 } }.
	\end{align*}
\end{proof}

\subsection{\pfref{thm:fast_rates}}
The proof largely follows \citep{koolen2016combining}, and thus for simplicity we only
show it for the expectation results.
We start from the regret bound:
\begin{align*}
	\reg(e_{\istar}) = \bigO{\ln(dT) + \sqrt{\ln(dT)\sumt\inner{w_t-e_{\istar}}{\ell_t-m'_t}^2}}.
\end{align*}
For the first result, by the condition, we have
\begin{align*}
	\E[\reg(e_{\istar})] = \E\sbr{\sumt\inner{w_t-e_{\istar}}{\ell_t}} =  \E\sbr{\sumt\sumi w_{t,i}\E_t[\ell_{t,i}-\ell_{t,\istar}]} \geq \Delta\E\sbr{\sumt\sum_{i\neq \istar} w_{t,i}}.
\end{align*}
On the other hand,
\begin{align*}
	\sumt\inner{w_t-e_{\istar}}{\ell_t-m'_t}^2 &= \sumt\rbr{ \sum_{i\neq \istar} w_{t,i}(\ell_{t,i}-m'_{t,i}-\ell_{t,\istar}+m'_{t,\istar}) }^2\\
	&\leq \sumt \sum_{i\neq \istar} w_{t,i}\rbr{ \ell_{t,i}-m'_{t,i}-\ell_{t,\istar}+m'_{t,\istar} }^2 \leq 16\sumt\sum_{i\neq \istar} w_{t,i}.
\end{align*}
Therefore,
\[
	\Delta\E\sbr{\sumt\sum_{i\neq \istar} w_{t,i}} \leq \E[\reg(e_{\istar})] = \bigO{ \sqrt{\ln(dT)\E\sbr{\sumt\sum_{i\neq\istar} w_{t,i}}} + \ln(dT) }.
\] 
Treating $\E\sbr{\sumt\sum_{i\neq \istar} w_{t,i}}$ as a variable and solving the inequality, we get $\E\sbr{\sumt\sum_{i\neq \istar} w_{t,i}}=\bigO{\ln(dT)/\Delta^2}$. Plugging this back we get $\E[\reg(e_{\istar})]=\bigO{\frac{\ln(dT)}{\Delta}}$.

For the second result, first note that by \pref{lem: stability for truncated one} we have
$w_{t,i}\in[\frac{1}{\sqrt{2}}w'_{t,i}, \sqrt{2}w'_{t,i}]$, $w'_{t,i}\in[\frac{1}{\sqrt{2}}w'_{t-1,i},  \sqrt{2}w'_{t-1,i}]$ and $w_{t-1,i}\in[\frac{1}{ \sqrt{2}}w'_{t-1,i},  \sqrt{2}w'_{t-1,i}]$, which implies $w_{t,i}\in [\frac{1}{2\sqrt{2}}w_{t-1,i}, 2\sqrt{2}w_{t-1,i}]$.
We then proceed as follows:
\begin{align*}
	\E[\reg(e_{\istar})] &
	= \bigO{ \E\sbr{\sqrt{\ln(dT)\sumt \inner{w_t-e_{\istar}}{\ell_t-\ell_{t-1}}^2 } } }\\
	&= \bigO{ \E\sbr{\sqrt{\ln(dT)\sumt \inner{w_t-e_{\istar}}{\ell_t}^2 + \inner{w_t-e_{\istar}}{\ell_{t-1}}^2 } } }\\
	&= \bigO{ \E\sbr{\sqrt{\ln(dT)\sumt \rbr{ \sumi w_{t,i}(\ell_{t,i}-\ell_{t,\istar}) }^2 +  \rbr{\sumi w_{t,i}(\ell_{t-1,i}-\ell_{t-1,\istar})}^2 } } }\\
	&= \bigO{ \E\sbr{\sqrt{\ln(dT)\sumt \rbr{ \sumi w_{t,i}(\ell_{t,i}-\ell_{t,\istar}) }^2 }} } \tag{$w_{t,i} \leq 2\sqrt{2}w_{t-1,i}$} \\
	&= \bigO{\sqrt{\ln(dT)\E\sbr{\sumt\sumi w_{t,i}(\ell_{t,i}-\ell_{t,\istar})^2 }}} \tag{Jensen's and 
Cauchy-Schwarz inequality}\\
	&= \bigO{ \sqrt{\frac{\ln(dT)}{\Delta}\E\sbr{\sumt\sumi w_{t,i}\E_t[\ell_{t,i}-\ell_{t,\istar}]^\kappa }} } \tag{by the assumption}\\
	&= \bigO{ \sqrt{ \frac{\ln(dT)}{\Delta}\E\sbr{\sumt\sumi w_{t,i}^{1-\kappa}\E_t[w_{t,i}(\ell_{t,i}-\ell_{t,\istar})]^\kappa }} }\\
	&= \bigO{ \sqrt{\frac{\ln(dT)}{\Delta}\E\sbr{\rbr{\sumt\sumi w_{t,i}}^{1-\kappa}\rbr{ \sumt\sumi\E_t[w_{t,i}(\ell_{t,i}-\ell_{t,\istar})] }^\kappa} } } \tag{H\"older's inequality}\\
	&= \bigO{ \sqrt{ \frac{\ln(dT)}{\Delta}T^{1-\kappa}\E[\reg(e_{\istar})]^\kappa } } \tag{Jensen's inequality}.
\end{align*}
Therefore, $\E[\reg(e_{\istar})]^{1-\kappa/2}=\bigO{\sqrt{\frac{\ln(dT)}{\Delta}T^{1-\kappa}}}$, and $\E[\reg(e_{\istar})]=\bigO{ \rbr{\frac{\ln(dT)}{\Delta}}^{\frac{1}{2-\kappa}}T^{\frac{1-\kappa}{2-\kappa}} }$.

\subsection{\pfref{thm:master}}
\begin{proof}
The regret $\reg(u) = \sumt \inner{w_t - u}{\ell_t}$ can be decomposed as the regret of base algorithm $\kstar$: $\reg_{\Astar}(u) = \sumt \inners{w_t^\kstar - u}{\ell_t}$, plus the regret of the master to this base algorithm:
$\sumt \inner{p_t - e_\kstar}{g_t} = \sumt \inners{w_t - w_t^\kstar}{\ell_t}$ (by the definition of $w_t$ and $g_t$).
It thus remains to apply the regret guarantee of \alg from \pref{lem:framework} (with $u$ in that lemma set to $e_\kstar$),
since the conditions of the lemma hold by the fact $g_{t,k}-h_{t,k} = \inner{w_t^k}{\ell_t-m_t}$.
The first term in \pref{eq:general_bound} becomes $\frac{1}{\etastar}\ln\frac{1}{p_{1, \kstar}'} + \sum_k \frac{p_{1, k}'}{\eta_k}$, which is $\frac{1}{\etastar}\ln\frac{\sum_k \eta_k^2}{\etastar^2} + \frac{\sum_k \eta_k}{\sum_k \eta_k^2}$ by the definition of $p_1'$.
The second term in \pref{eq:general_bound} is simply zero since the learning rate stays the same over time.
The third term equals $32\etastar \sumt \inners{w_t^\kstar}{\ell_t - m_t}^2$.
Dropping the last negative term then finishes the proof.
\end{proof}

%% file: appendix_expert.tex

\subsection{\pfref{thm:KL}}
\begin{proof}
By the construction, for any $u$ there exists $\kstar$ such that $\eta_{\kstar}\leq \min\left\{\frac{1}{64}, \sqrt{\frac{\KL(u, \pi)+\ln V(u)}{V(u)}} \right\} \leq 2\eta_{\kstar}$.
	Therefore, from \pref{eq:KL_base_alg_bound} we have:
	\begin{align}
		&\reg_{\calA_{\kstar}}(u) \leq \frac{\KL(u, \pi)}{2\eta_{\kstar}} + 64\eta_{\kstar}\sumt\sumi u_i(\ell_{t,i}-m_{t,i})^2 - 32\eta_{\kstar}\sumt\sumi w_{t,i}^{\kstar}(\ell_{t,i} - m_{t,i})^2 \notag\\
		&\leq \bigO{\KL(u, \pi) + \sqrt{(\KL(u, \pi)+\ln V(u))V(u)} } - 32\eta_{\kstar}\sumt\sumi w_{t,i}^{\kstar}(\ell_{t,i} - m_{t,i})^2. \label{eq:KL_base_alg_bound2}
	\end{align}
	Further note that $32\eta_k\left|  \inner{w^k_t}{\ell_t-m_t} \right|\leq 32\eta_k\norm{\ell_t-m_t}_{\infty}\leq 1$.
	Hence, we apply \pref{thm:master} with $\sum_k\eta_k=\Theta(1), \sum_k\eta_k^2=\Theta(1), \frac{\sum_k\eta_k^2}{\eta^2_\kstar}=\bigO{ 1/\eta^2_{\kstar} }=\bigO{ \frac{V(u)}{\KL(u, \pi)+\ln V(u)} }=\bigO{V(u)}$,
	and cancel the last term in \pref{eq:master_bound} by the last negative term in \pref{eq:KL_base_alg_bound2} via Cauchy-Schwarz inequality, arriving at
	\begin{align*}
		\REG(u) &\leq \bigO{\KL(u, \pi) + \sqrt{(\KL(u, \pi)+\ln V(u))V(u)} } + \frac{1}{\eta_{\kstar}}\ln\frac{\sum_k\eta^2_k}{\eta^2_\kstar}\\
		&= \bigO{ \KL(u, \pi) + \ln V(u) + \sqrt{(\KL(u, \pi)+\ln V(u))V(u)} }
	\end{align*}
	and finishing the proof.
\end{proof}

\subsection{\pfref{thm:MS}}
\begin{proof}
	By the definition of $\calS$, it is clear that $|\calS|$ is at most $\bigO{d\log_2T}$ so our algorithm is efficient.
	For any $\istar\in[d]$, there exists a $\kstar$ such that $\eta_{\kstar} \leq \min\left\{\frac{1}{128c_{\istar}}, \sqrt{\frac{\Gamma_{\istar}}{\sumt(\ell_{t,\istar}-m_{t,\istar})^2}}\right\} \leq 2\eta_{\kstar}$.
	Moreover, $32\cdot \eta_{t,i}| \ell_{t,i}-m_{t,i} |\leq 128\eta_kc_i\leq 1$ for all $i \in \calZ(k)$.
	Hence, the conditions of \pref{lem:framework} hold, and with $|\calZ(k)|\leq d$ we have
	\begin{align*}
		\reg_{\calA_{\kstar}}(e_{\istar}) &\leq \bigO{\frac{\ln d}{\eta_{\kstar}}} + 64\eta_{\kstar}\sumt(\ell_{t,\istar}-m_{t,\istar})^2 - 32\eta_{\kstar}\sumt\sumi w_{t,i}(\ell_{t,i}-m_{t,i})^2\\
		&= \bigO{ c_{\istar}\Gamma_{\istar} + \sqrt{\Gamma_{\istar}\sumt(\ell_{t,\istar}-m_{t,\istar})^2} } - 32\eta_{\kstar}\sumt\sumi w_{t,i}(\ell_{t,i}-m_{t,i})^2.
	\end{align*}
	Next, also note that the conditions of \pref{thm:master} hold since \[32\eta_k |\inner{w_t^k}{\ell_t - m_t}| \leq 64\eta_k \max_{i\in\calZ(k)}c_i \leq 1.\]
	Thus, with the last negative term from the bound for $\reg_{\calA_{\kstar}}(e_{\istar})$ above canceling the last term of  \pref{eq:master_bound}, and $\sum_k\eta_k=\Theta(1/\cmin)$, $\sum_k\eta_k^2=\Theta(1/\cmin^2)$, and $\frac{\sum_k\eta_k^2}{\eta_{\kstar}^2}=\bigO{\frac{c_{\istar}^2T}{\cmin^2}}$, we obtain:
	\begin{align*}
		\reg(e_{\istar}) &= \bigO{ c_{\istar}\Gamma_{\istar} + \sqrt{\Gamma_{\istar}\sumt(\ell_{t,\istar}-m_{t,\istar})^2} + \frac{1}{\eta_{\kstar}}\Gamma_{\istar} + \cmin }\\
		&= \bigO{ c_{\istar}\Gamma_{\istar} + \sqrt{\Gamma_{\istar}\sumt(\ell_{t,\istar}-m_{t,\istar})^2} },
	\end{align*}
	which completes the proof.
\end{proof}

\input{appendix-separate-switching-regret}

\subsection{Omitted details for \pref{subsec:unknown_range}}\label{app:unknown_range}

Since some of the results in this section will be used later for OLO as well, we use $\dnorm{\cdot}$ to denote $L_\infty$ norm in the context of an expert problem and $L_2$ norm in the context of an OLO problem.

We apply a variant of the techniques introduced in \cite{cutkosky2019artificial} to deal with unknown loss range.
We start with an initial guess $B_0$ on the range of $\max_t\dnorm{\ell_t-m_t}$.
Denote by $B_t=\max_{0\leq s\leq t}\dnorm{\ell_s-m_s}$ the range of predicted error up to episode $t$, and $B=B_T$. 
We feed the following truncated loss to the algorithm in each episode:
\begin{equation}
	\bar{\ell}_t=m_t + \frac{B_{t-1}}{B_t}(\ell_t-m_t).\label{eq:truncated_loss}
\end{equation}
Note that $\dnorm{\bar{\ell}_t-m_t}\leq B_{t-1}$.
Thus, the truncated loss allows the learner to assume that the range of predicted error in episode $t$ is known at the beginning of this episode.
Doing so already gives an algorithm that can deal with unknown loss range when $B_T/B_0$ is not too big.
To further deal with arbitrary ratio $B_t/B_0$, we also incorporate a restarting scheme which is a simplified version of that in \cite{mhammedi2019lipschitz}.
The restarting scheme makes sure the learning rate can always be properly tuned and replace the potential $\ln(B_T/B_0)$ dependency by $\ln T$.
We summarize ideas above as a new master algorithm in \pref{alg:unknown_range}, which requires an expert set generator $\calE$ as input.
The expert set generator $\calE$ is a function that maps any initial guess $B_0$ to a set of (learning rate, base algorithm) pairs $\calE(B_0)$.

To obtain data-dependent bound in expert problem with unknown range, it suffices to run \pref{alg:unknown_range} with the following expert set generator:
\begin{equation}
\begin{split}\label{eq:expert_unknown_range}
	\URbase(B_0) &= \Big\{ (\eta_k, \calA_k): \forall k =1, \ldots, N, \eta_k = \tfrac{1}{32B_0\cdot 2^{k}}, \text{$\calA_k$ is \alg with $w_1' = \pi$,} \\
	&\qquad\qquad \text{$\Omega_t = \Delta_d$, and $\eta_{t,i}=2\eta_k$ for all $t$ and $i$} \Big\},
\end{split}
\end{equation}
and $\Lambda_t=\Delta_{\calS_{\text{\rm UR}}(t)}$, where $\calS_{\text{\rm UR}}(t)=\{k: \tfrac{1}{32B_0\cdot 2^{k}}\leq\frac{1}{64B_{t-1}}\}$, and $N=\lceil\log_2 (2T^2) \rceil$.

\setcounter{AlgoLine}{0}
\begin{algorithm}[t]
    \caption{\master with unknown loss range}
    \label{alg:unknown_range}
    \textbf{Input:} An expert set generator $\calE$, initial scale $B_0$.
    	
	\textbf{Initialize:} $\tilB=B_0$, $\calA$ as an instance of \pref{alg:master} with input $\calE(\tilB)$.
	
	\For{$t=1,\ldots, T$}{
		Obtain decision $w_t$ from $\calA$, play $w_t$.
	
		Receive loss $\ell_t$, and feed $\bar{\ell}_t=m_t + \frac{B_{t-1}}{B_t}(\ell_t-m_t)$ to $\calA$, where $B_t=\max_{0\leq s\leq t}\dnorm{\ell_s-m_s}$.
		
		\If{$B_t/\tilB > T$}{
			$\tilB = B_t$.
		
			Reset $\calA$ as a new instance of \pref{alg:master} with input $\calE(\tilB)$.
		}
	}
\end{algorithm}

%

\begin{proof}[of \pref{thm:unknown_range}]
	Define $\bar{V}(u) = \max\cbr{3, \sumt\sumi u_i(\bar{\ell}_{t,i}-m_{t,i})^2}$.
	We first show that the desired bound holds when there is no restart before episode $T$, that is, $\frac{B_{T-1}}{B_0}\leq T$.
	In this case, $\frac{1}{\max\{64, T\}B_{T-1}}\geq\frac{1}{32B_0\cdot 2^N}$.
	Hence, there exists $k_\star\in\URbase(B_0)$ such that
	\[
	 	\eta_{k_\star}\leq \min\left\{\frac{1}{64B_{T-1}}, \sqrt{\frac{\KL(u, \pi) + \ln T}{\bar{V}(u)}}\right\} \leq 2\eta_{k_\star}.
	 \]
	The conditions of \pref{lem:framework} hold since $32\cdot 2\eta_{k_\star}\norm{\bar{\ell}_t - m_t}_{\infty}\leq 64\eta_{k_\star} B_{T-1} \leq 1$ for any $t\leq T$. We thus have
	\begin{align*}
		&\sumt\inner{w^{k_\star}_t - u}{\bar{\ell}_t} = \frac{\KL(u, \pi)}{2\eta_{k_\star}} + 64\eta_{k_\star}\sumt\sum_{i=1}^du_i(\bar{\ell}_{t, i}-m_{t, i})^2 - 32\eta_{k_\star}\sumt\sum_{i=1}^dw^{\kstar}_{t, i}(\bar{\ell}_{t, i}-m_{t, i})^2\\
		&= \bigO{\sqrt{ (\KL(u, \pi) + \ln T)\bar{V}(u) } + B\KL(u, \pi) } - 32\eta_{k_\star}\sumt\sumi w^{\kstar}_{t, i}(\bar{\ell}_{t, i}-m_{t, i})^2.
	\end{align*}
	Note that $k_\star\in\calS_{\text{\rm UR}}(T)$, and for any $k\in\calS_{\text{\rm UR}}(t)$,
	\begin{align*}
		32\eta_k\left| \inner{w^k_t}{\bar{\ell}_t-m_t} \right| \leq 32\eta_k\norm{\bar{\ell_t}-m_t}_{\infty} \leq 1.
	\end{align*}
	Hence, the conditions of \pref{thm:master} also hold, and with $\sum_k\eta_k = \Theta(\frac{1}{B_0}), \sum_k\eta^2_k=\Theta(\frac{1}{B_0^2})$, $\frac{\sum_k\eta_k^2}{\eta^2_{k_\star}}=\bigO{(\eta_1/\eta_{k_\star})^2} = \bigO{\ln T}$ by $2^N=\bigo{T^2}$ and $\inner{w^{\kstar}_t}{\bar{\ell}_t-m_t}^2\leq \sumi w^{\kstar}_{t, i}(\bar{\ell}_{t, i}-m_{t, i})^2$, we have:
	\begin{align*}
		\sumt\inner{w_t - u}{\bar{\ell_t}} &\leq \sumt\inner{w^{k_\star}_t - u}{\bar{\ell}_t} + \bigO{\frac{1}{\eta_{k_\star}}\ln T + B_0} + 32\eta_{k_\star}\sumt\inner{w^{k_\star}_t}{\bar{\ell}_t - m_t}^2\\
		&\leq \bigO{ \sqrt{ (\KL(u, \pi) + \ln T)\bar{V}(u) } } + B(\KL(u, \pi) + \ln T).
	\end{align*}
	Moreover, since $\ell_t-\bar{\ell}_t=\frac{B_t-B_{t-1}}{B_t}(\ell_t-m_t)$, the difference between the regret measured with $\ell_t$ and that with $\bar{\ell}_t$ is
	\begin{align*}
		\sumt\inner{w_t-u}{\ell_t - \bar{\ell}_t} &\leq 2\sumt\norm{\ell_t-\bar{\ell_t}}_{\infty} \leq 2\sumt\frac{B_t-B_{t-1}}{B_t}\norm{\ell_t-m_t}_{\infty}\\
		&\leq 2\sumt (B_t - B_{t-1}) \leq 2B.
	\end{align*}
	Therefore, noticing $\bar{V}(u)\leq V(u)$, we prove the desired result when there is no restart:
	\begin{align*}
		\sumt\inner{w_t - u}{\ell_t} &= \sumt\inner{w_t-u}{\bar{\ell}_t} + \sumt\inner{w_t-u}{\ell_t-\bar{\ell}_t}\\
		&= \bigO{ \sqrt{ (\KL(u, \pi) + \ln T)\bar{V}(u) }  + B(\KL(u, \pi) + \ln T) }.
	\end{align*}
	Next, we show that the desired bound also holds when there is at least one restart before episode $T$.
	Denote by $\tau_2$ the episode of the last restart, and by $\tau_1$ the episode of the second last restart ($\tau_1=0$ if the algorithm only restarts once).
	We consider the regret in the following three intervals: $[1, \tau_1], (\tau_1, \tau_2], (\tau_2, T]$.
	Denote $y_t=\norm{\ell_t-m_t}_{\infty}$.
	For regret in $[1, \tau_1]$, we have:
	\begin{align}
		\label{eq:first restart interval}
		\reg^{[1, \tau_1]}(u) \leq \sum_{t=1}^{\tau_1}y_t = B_{\tau_1}\sum_{t=1}^{\tau_1}\frac{y_t}{B_{\tau_1}} \leq B_{\tau_1}T < B_{\tau_1}\frac{B_{\tau_2}}{B_{\tau_1}} \leq B_T,
	\end{align}
	where we apply $T<B_{\tau_2}/B_{\tau_1}$ due to the restart condition.
	Within intervals $(\tau_1, \tau_2]$ and $(\tau_2, T]$, we have $B_{\tau_2-1}/B_{\tau_1}\leq T, B_{T-1}/B_{\tau_2} \leq T$. Thus, by the regret guarantee with restart only at the end of an interval,
	\begin{align*}
		\reg^{(\tau_1, \tau_2]}(u) &= \bigO{B_{\tau_2}(\KL(u, \pi)+\ln T ) + \sqrt{(\KL(u, \pi)+\ln T) V^{(\tau_1, \tau_2]}(u)}}\\
		\reg^{(\tau_2, T]}(u) &= \bigO{B_T(\KL(u, \pi)+\ln T) + \sqrt{(\KL(u, \pi)+\ln T) V^{(\tau_2, T]}(u)}},
	\end{align*}
	where $V^{\calI}(u) = \max\cbr{3, \sum_{t\in\calI}\sumi u_i(\ell_{t,i}-m_{t,i})^2}$.
	Summing all three regret bounds together and applying the Cauchy-Schwarz inequality, we get the desired result.
\end{proof}

%% file: appendix-separate-switching-regret.tex
\subsection{\pfref{thm:switch}}
\begin{proof}
     We first focus on a specific $j$ and bound the regret within $\calI_j$. The regret in this interval can be decomposed as 
     \begin{align*}
          \sum_{t\in\calI_j} \inner{w_t - u_j}{\ell_t} 
          &= \sum_{t\in\calI_j} \inner{w_t - w_t^{r}}{\ell_t} + \sum_{t\in\calI_j} \inner{w_t^{r} - u_j}{\ell_t} \\
          &= \sum_{t\in\calI_j} \inner{p_t - e_{r}}{g_t} + \sum_{t\in\calI_j} \inner{w_t^{r} - u_j}{\ell_t}\\
          &\leq \sum_{t\in\calI_j} \inner{p_t - \overline{e}_{r}}{g_t} + \sum_{t\in\calI_j} \inner{w_t^{r} - u_j}{\ell_t} + \order(1)    \tag{define $\overline{e}_r=(1-\frac{1}{T})e_r + \frac{\one}{\lceil\log_2 T\rceil T}$}
     \end{align*}
     for any $r\in [\lgT]$. 
     
     The term $\sum_{t\in\calI_j} \inner{w_t^{r} - u_j}{\ell_t}$ corresponds to the regret of the $r$-th base algorithm in the interval $\calI_j$. Let $s_j$ be the first time index in $\calI_j$, and recall that the $r$-th expert is an \alg with a fixed learning rate $2\eta_r$, and a feasible set $\Omega_t=\{w\in\Delta_d: w_i\geq \frac{1}{dT}\}$. To upper bound it, we follow the exact same arguments as in the proof of \pref{lem:framework}, except for replacing the summation range $[1, T]$ with $\calI_j$. 
     This leads to:
	\begin{align*}
		&\sum_{t\in\calI_j}\inner{w_t^r-u_j}{\ell_t} \\
		&\leq \frac{1}{2\eta_r} \sumi \kl(u_{j,i}, w_{s_j, i}^{r\prime})
		 +  32\sum_{t\in\calI_j}\sumi 2 \eta_r u_{j,i}(\ell_{t,i}-m_{t,i})^2 - 16\sum_{t\in\calI_j}\sumi 2\eta_r w^r_{t,i}(\ell_{t,i}-m_{t,i})^2 \\
		 &= \frac{1}{2\eta_r}\sumi u_{j,i}\ln\frac{u_{j,i}}{w_{s_j,i}^{r\prime}} +  32\sum_{t\in\calI_j}\sumi 2 \eta_r u_{j,i}(\ell_{t,i}-m_{t,i})^2 - 16\sum_{t\in\calI_j}\sumi 2\eta_r w^r_{t,i}(\ell_{t,i}-m_{t,i})^2 \\
		 &\leq  \frac{1}{2\eta_r}\ln(dT)	      +  32\sum_{t\in\calI_j}\sumi 2 \eta_r u_{j,i}(\ell_{t,i}-m_{t,i})^2 - 16\sum_{t\in\calI_j}\sumi 2\eta_r w^r_{t,i}(\ell_{t,i}-m_{t,i})^2.
	\end{align*}

	Next, we deal with $\sum_{t\in\calI_j} \inner{p_t - \overline{e}_{r}}{g_t}$. Recall that \master uses a regularizer $\psi(p)=\sum_{k=1}^{\lgT} \frac{1}{\eta_k}p_k \ln p_k$. Again, similarly to the proof of \pref{lem:framework}, considering the regret only in $\calI_j$ and dropping the negative term, we have
	\begin{align*}
	     \sum_{t\in\calI_j} \inner{p_t - \overline{e}_{r}}{g_t} 
	     &\leq \sum_{t\in\calI_j} \left( D_{\psi}(\overline{e}_r, p_t') -  D_{\psi}(\overline{e}_r, p_{t+1}')\right) + 32\sum_{t\in\calI_j}\sum_{k=1}^{\lgT}\eta_{k}\overline{e}_{r,k}(g_{t,k}-h_{t,k})^2  
	     \\
	     &\leq D_{\psi}(\overline{e}_r, p_{s_j}') -  D_{\psi}(\overline{e}_r, p_{s_{j+1}}') + 32\eta_{r}\sum_{t\in\calI_j}\inner{w^r_t}{\ell_t-m_t}^2  + \order(1),
	\end{align*}
	where $s_{j+1}$ is defined as $T+1$ if $j$ is the last interval.
	We further deal with the first term above:  
	\begin{align*}
	     &D_{\psi}(\overline{e}_r, p_{s_j}') -  D_{\psi}(\overline{e}_r, p_{s_{j+1}}') \\
	     &= \sum_{k=1}^{\lgT}\frac{1}{\eta_k}\left( \overline{e}_{r,k}\ln \frac{p_{s_{j+1},k}'}{p_{s_j,k}'} + p_{s_j,k}' - p_{s_{j+1},k}' \right) \\
	     &\leq \frac{\ln(\lgT T)}{\eta_r} + \sum_{k=1}^{\lgT}\frac{1}{\eta_k}\left( p_{s_j,k}' - p_{s_{j+1},k}' \right) + \order(\ln(dT)). 
	\end{align*}
	Combining all bounds above, we get that for any $r\in\lgT$: 
	\begin{align*}
	      &\sum_{t\in\calI_j} \inner{w_t - u_j}{\ell_t} \\
	      &\leq  \frac{1}{2\eta_r} \ln(dT)
		 +  32\sum_{t\in\calI_j}\sumi 2 \eta_r u_{j,i}(\ell_{t,i}-m_{t,i})^2 - 16\sum_{t\in\calI_j}\sumi 2\eta_r w^r_{t,i}(\ell_{t,i}-m_{t,i})^2 \\
		 &\qquad \qquad +\frac{\ln(\lgT T)}{\eta_r} + \sum_{k=1}^{\lgT}\frac{1}{\eta_k}\left( p_{s_j,k}' - p_{s_{j+1},k}' \right) +  32\eta_{r}\sum_{t\in\calI_j}\inner{w^r_t}{\ell_t-m_t}^2 + \order(\ln(dT))\\
		 &\leq \order\left(\frac{\ln(dT)}{\eta_r} + \eta_r \sum_{t\in\calI_j}\sumi u_{j,i}(\ell_{t,i}-m_{t,i})^2 \right) + \order(\ln(dT)) + \sum_{k=1}^{\lgT}\frac{1}{\eta_k}\left( p_{s_j,k}' - p_{s_{j+1},k}' \right)
	\end{align*}
	where we use Jenson's inequality: $\inner{w^r_t}{\ell_t-m_t}^2\leq \sumi w^r_{t,i}(\ell_{t,i}-m_{t,i})^2$.  Specifically, applying the above bound with the $r$ such that 
	\begin{align*}
	     \eta_{r}\leq \min\left\{\frac{1}{64}, \sqrt{\frac{\ln(dT)}{\sum_{t\in\calI_j} \sum_{i=1}^d u_{j,i}(\ell_{t,i}-m_{t,i})^2}} \right\} \leq 2\eta_{r}, 
	\end{align*}
	we get 
	\begin{align}
	    \sum_{t\in\calI_j} \inner{w_t - u_j}{\ell_t} = \order\left(\sqrt{\ln(dT) \sum_{t\in\calI_j} \sum_{i=1}^d u_{j,i}(\ell_{t,i}-m_{t,i})^2} + \ln(dT)\right) +  \sum_{k=1}^{\lgT}\frac{1}{\eta_k}\left( p_{s_j,k}' - p_{s_{j+1},k}' \right) \label{eq:interval_regret_telescoping}. 
	\end{align}
	Finally, summing the above bound over $j=1,2,\ldots, S$ and telescoping, we get 
	\begin{align*}
	     \sum_{j=1}^S \sum_{t\in\calI_j} \inner{w_t - u_j}{\ell_t} 
	     &=  \order\left(\sum_{j=1}^S \sqrt{\ln(dT) \sum_{t\in\calI_j} \sum_{i=1}^d u_{j,i}(\ell_{t,i}-m_{t,i})^2} + S\ln(dT)\right) +  \sum_{k=1}^{\lgT}\frac{ p_{1,k}'}{\eta_k} \\
	     &=  \order\left(\sum_{j=1}^S \sqrt{\ln(dT) \sum_{t\in\calI_j} \sum_{i=1}^d u_{j,i}(\ell_{t,i}-m_{t,i})^2} + S\ln(dT)\right), 
	\end{align*}
	finishing the proof. 
\end{proof}

Note that importantly, the last term in \pref{eq:interval_regret_telescoping} only disappears (mostly) after summed over all intervals.
As mentioned, getting an interval regret bound like \pref{eq:interval_regret_telescoping} but without the last term is impossible, proven in the next section.

\subsection{Impossible results for interval regret}\label{app:impossible_interval}
\begin{theorem}
     For a two-expert problem with loss range $[-1, 1]$, it is impossible to achieve the following regret bound for all interval $\calI\subseteq [1, T]$ and all comparators $i\in\{1,2\}$ simultaneously: 
     \begin{align*}
          \sum_{t\in\calI} \inner{p_t-e_i}{\ell_t} = \tilde{\order}\left(\sqrt{\sum_{t\in\calI} |\ell_{t,i}|} +1\right). 
     \end{align*}
\end{theorem}
\begin{proof}
     Consider an envinronment where the losses of Expert 1 is a deterministic value $\ell_{t,1}=0$, and the losses of Expert 2 are i.i.d. chosen in each round according to the following: 
     \begin{align*}
           \ell_{t,2} = \begin{cases}
                 1   &\text{with probability\ } \frac{1}{2} - \epsilon \\
                 -1  &\text{with probability\ } \frac{1}{2} + \epsilon
           \end{cases}
     \end{align*}
     where $\epsilon=T^{-\frac{1}{5}}$. We assume that $\epsilon\leq \frac{1}{4}$ (which is equivalent to assuming $T\geq 4^5$). For simplicity, we call this distribution $\calD$. 
     Note that the expected loss of Expert 2 is $-2\epsilon$, smaller than that of Expert 1. Therefore, in this environment, the expected regret of the learner during $[1, T]$ would be 
     \begin{align*}
         \E[\reg^{[1,T]}(e_2)] = 2\epsilon \E\left[\sum_{t=1}^T p_{t,1}\right]. 
     \end{align*}
     Define $L=T^{\frac{3}{10}}$, and divide the whole horizon into $\frac{T}{L}=T^{\frac{7}{10}}$ intervals. Denote them as $\calI_k=\{(k-1)L+1, \ldots, kL\}$ for $k=1,2,\ldots, \frac{T}{L}$. Let 
     \begin{align*}
           k^\star = \argmin_k \E\left[\sum_{t\in\calI_k} p_{t,1}\right]. 
     \end{align*}
     That is, $k^\star$ is the interval where the learner would put least weight on Expert 1 in expectation. We then create another environment, where the loss of Expert 2 is same as the previous environment in interval $1, 2, \ldots, k^\star-1$, but change to the following starting from interval $k^\star$: 
     \begin{align*}
           \ell_{t,2} = \begin{cases}
                 1   &\text{with probability\ } \frac{1}{2} + \epsilon \\
                 -1  &\text{with probability\ } \frac{1}{2} - \epsilon
           \end{cases}
     \end{align*} 
     We call this distribution $\calD'$. 
     In this alternative environment, starting from interval $k^\star$, the best expert becomes Expert 1, and the expected interval regret of the learner is 
     \begin{align}
          \E'[\reg^{\calI_{k^\star}}(e_1)] = 2\epsilon \E'\left[\sum_{t\in\calI_{k^\star}} p_{t,2}\right]=2\epsilon L - 2\epsilon \E'\left[\sum_{t\in\calI_{k^\star}} p_{t,1}\right]. \label{eq: tmp100}
     \end{align}
     where we use $\E'[\cdot]$ to denote the expectation under this alternative environment. 
     
     Below we denote the probability measure under the two environments as $\calP$ and $\calP'$ respectively. Since $p_{t,1}$ is a function of $\{\ell_{\tau}\}_{\tau=1}^{t-1}$, by standard arguments, 
     \begin{align*}
          & \E'\left[\sum_{t\in\calI_{k^\star}} p_{t,1}\right]  - \E\left[\sum_{t\in\calI_{k^\star}} p_{t,1} \right]  \\
          &\leq L\left\|\calP(\{\ell_\tau\}_{\tau=1,\ldots, k^\star L}) -  \calP'(\{\ell_\tau\}_{\tau=1,\ldots, k^\star L}) \right\|_{\text{TV}}  \tag{$\|\cdot\|_{\text{TV}}$ is the total variation}\\
          &\leq \frac{L}{2}\sqrt{\KL(\calP(\{\ell_\tau\}_{\tau=1,\ldots, k^\star L}),~ \calP'(\{\ell_\tau\}_{\tau=1,\ldots, k^\star L}))}    \tag{Pinsker's inequality} \\
          &= \frac{L}{2}\sqrt{L\KL(\calD,~ \calD')} \\
          &= \frac{L^{\frac{3}{2}}}{2}\sqrt{\left(\frac{1}{2}+\epsilon\right) \ln \frac{\frac{1}{2}+\epsilon}{\frac{1}{2}-\epsilon} + \left(\frac{1}{2}-\epsilon\right)\ln \frac{\frac{1}{2}-\epsilon}{\frac{1}{2}+\epsilon}} \\
          &\leq \frac{L^{\frac{3}{2}}}{2}\sqrt{2\epsilon \ln \frac{\frac{1}{2}+\epsilon}{\frac{1}{2}-\epsilon}} \leq \frac{L^{\frac{3}{2}}}{2}\sqrt{\frac{4\epsilon^2}{\frac{1}{2}-\epsilon}}\leq 2L^{\frac{3}{2}}\epsilon, 
     \end{align*}
     where we use $\ln(1+\alpha)\leq \alpha$ and $\epsilon\leq \frac{1}{4}$. Notice that $\frac{L}{T}\frac{1}{2\epsilon}\E[\reg^{[1,T]}(e_2)]= \frac{L}{T}\E\left[\sum_{t=1}^T p_{t,1}\right]\geq  \E\left[\sum_{t\in\calI_{k^\star}} p_{t,1}\right]$ by the definition of $k^\star$, and $\E'\left[ \sum_{t\in\calI_{k^\star}} p_{t,1} \right] = L - \frac{\E'[\reg^{\calI_{k^\star}}(e_1)]}{2\epsilon}$ by \pref{eq: tmp100}. Using them in the above inequality, we get 
     \begin{align*}
          L - \frac{\E'[\reg^{\calI_{k^\star}}(e_1)]}{2\epsilon} - \frac{L}{2\epsilon T}\E[\reg^{[1,T]}(e_2)] \leq 2L^{\frac{3}{2}}\epsilon.    
     \end{align*}
     Using the values we choose, this is equivalent to 
     \begin{align*}
          T^{\frac{3}{10}} - \frac{T^{\frac{2}{10}}}{2}\E'[\reg^{\calI_{k^\star}}(e_1)] - \frac{1}{2T^{\frac{5}{10}}} \E[\reg^{[1,T]}(e_2)]  \leq 2T^{\frac{1}{4}}. 
     \end{align*}
     When $T$ is large enough, we see that either $\E[\reg^{[1,T]}(e_2)] \geq \Omega(T^{\frac{8}{10}})$ or $\E'[\reg^{\calI_{k^\star}}(e_1)] \geq \Omega(T^{\frac{1}{10}})$. However, the desired bound $\sqrt{\sum_{t\in\calI}|\ell_{t,i}|}$ is $\order(\sqrt{T})$ and $\order(1)$ in the two cases respectively. One of them must be violated, thus the desired bound is impossible.
\end{proof}

%% file: appendix_OLO.tex


In this section, when presenting the base algorithms, we sometimes use $\Omega$ as its decision set, which should be seen as a subset of $\calK$ (thus its size is bounded by $D$ as well) and will be set appropriately by the master.

\subsection{Combining Online Newton Step}\label{app:ONS}



We first introduce a variant of the ONS algorithm (\pref{alg:ONS}) and present its regret guarantee.
To make the algorithm general enough to deal with unknown loss range in \pref{app:unconstrained},
we consider a slightly more general setup where the algorithm receives a \textit{range hint} $z_t$ at the beginning of round $t$, which is guaranteed to satisfy $\norm{\ell_t-m_t} \leq z_t$.
For this section and the result of \pref{thm:ONS}, it suffices to set $z_t=1$ for all $t$.
The guarantee of this ONS variant is as follows.

\setcounter{AlgoLine}{0}
\begin{algorithm}[t]
    \caption{A Variant of Online Newton Step}
    \label{alg:ONS}
    \textbf{Parameters:} learning rate $\eta>0$, $w'_1=\vec{0}$.
    
    \textbf{Define:} $c_t(w) = \inner{w}{\ell_t}+16\eta\inner{w}{\ell_t - m_t}^2$ and
    $\nabla_t =\nabla c_t(w_t)=\ell_t + 32\eta\inner{w_t}{\ell_t-m_t}(\ell_t-m_t)$.

    \For{$t=1$,\ldots, T}{
		Receive prediction $m_t$ and range hint $z_t$.
    
    		Update $w_t = \argmin_{w\in\Omega}\left\{\inner{w}{m_t} + D_{\psi_{t}}(w,w_t')\right\}$ where $\psi_t(w) = \frac{1}{2}\norm{w}_{A_t}^2$ and
      \begin{align*}
      A_t &=  \eta\rbr{4z_1^2\cdot I + \sum_{s=1}^{t-1} (\nabla_s-m_s)(\nabla_s-m_s)^\top + 4z_t^2\cdot I }.
    \end{align*}
    
    		Receive $\ell_t$.
    
        Update $w_{t+1}' = \argmin_{w\in\Omega}\left\{\inner{w}{\nabla_t} + D_{\psi_t}(w, w_t')\right\}$.
    }
\end{algorithm}

\begin{lemma}\label{lem:ONS}
	Suppose $\norm{\ell_t-m_t}_2\leq z_t$, $z_t$ is non-decreasing in $t$, and $64\eta Dz_T \leq 1$. Then \pref{alg:ONS} ensures for any $u\in\Omega$ (with $r$ being the rank of $\calL_T =\sumt(\ell_t-m_t)(\ell_t-m_t)^\top$)
    \begin{align*}
        &\sum_{t=1}^T \inner{w_t-u}{\ell_t}\\
        &\leq \bigO{\frac{r\ln(Tz_T/z_1)}{\eta} + z_1\norm{u}_2 + D(z_T-z_1) + \eta\sum_{t=1}^T \inner{u}{\ell_t-m_t}^2 } - 11\eta\sum_{t=1}^T \inner{w_t}{\ell_t-m_t}^2.
    \end{align*}
\end{lemma}
\begin{proof}
	By \pref{lem:oomd} and \pref{lem:olo stabability}, we have:
	\begin{align*}
		\sum_{t=1}^T\inner{w_t - u}{\nabla_t} &\leq \sum_{t=1}^T\inner{w_t - w'_{t+1}}{\nabla_t - m_t} + D_{\psi_t}(u, w'_t) - D_{\psi_t}(u, w'_{t+1})\\
		&\leq 2\sum_{t=1}^T\norm{\nabla_t -m_t}_{A_t^{-1}}^2 + D_{\psi_1}(u, w'_1) + \sum_{t=1}^{T-1}D_{\psi_{t+1}}(u, w'_{t+1}) - D_{\psi_t}(u, w'_{t+1})\\
		&\leq \bigO{\frac{r\ln(Tz_T/z_1)}{\eta} + \eta z_1^2\norm{u}_2^2} + \sum_{t=1}^{T-1}D_{\psi_{t+1}}(u, w'_{t+1}) - D_{\psi_t}(u, w'_{t+1}).
	\end{align*}
	Note that $\eta z_1^2\norm{u}_2^2\leq \eta Dz_T z_1\norm{u}_2=\bigO{z_1\norm{u}_2}$. Moreover,
	\begin{align*}
		&\sum_{t=1}^{T-1}D_{\psi_{t+1}}(u, w'_{t+1}) - D_{\psi_t}(u, w'_{t+1})\\
		&\leq \frac{\eta}{2}\sum_{t=1}^{T-1}\inner{u-w'_{t+1}}{\nabla_t - m_t}^2 + \bigO{\eta D^2\sum_{t=1}^{T-1}(z_{t+1}^2 - z_t^2)} \\
		&\leq \eta\sum_{t=1}^{T-1}\inner{u-w_t}{\nabla_t-m_t}^2 + \eta\sum_{t=1}^{T-1}\inner{w_t-w'_{t+1}}{\nabla_t-m_t}^2 + \bigO{\eta D^2 z_T\sum_{t=1}^{T-1}(z_{t+1} - z_t)}\\
		&\leq 2\eta\sumt\inner{u}{\nabla_t-m_t}^2 + 2\eta\sumt\inner{w_t}{\nabla_t-m_t}^2 + \bigO{\frac{r\ln(Tz_T/z_1)}{\eta}} + \bigO{D(z_T-z_1)} \tag{by $0\leq\eta\inner{w_t-w'_{t+1}}{\nabla_t-m_t}\leq 3\eta Dz_t = \order(1)$ and \pref{lem:olo stabability}}\\
		&\leq 5\eta\sum_{t=1}^{T}\inner{u}{\ell_t-m_t}^2 + 5\eta\sum_{t=1}^{T}\inner{w_t}{\ell_t-m_t}^2 + \bigO{\frac{r\ln(Tz_T/z_1)}{\eta}} + \bigO{D(z_T-z_1)}. \tag{by definition of $\nabla_t$ and $32\eta\abs{\inner{w_t}{\ell_t-m_t}}\leq 32\eta D z_t \leq \frac{1}{2}$}
	\end{align*}
	Since $c_t(w)$ is convex in $w$, we have
	\begin{align*}
		\sumt c_t(w_t) - c_t(u) &= \sumt\inner{w_t-u}{\ell_t} + 16\eta\sumt\inner{w_t}{\ell_t-m_t}^2-\inner{u}{\ell_t-m_t}^2 \leq \sum_{t=1}^T\inner{w_t - u}{\nabla_t}.
	\end{align*}
	Reorganizing terms then finishes the proof.
\end{proof}

\begin{lemma}\label{lem:olo stabability}
In \pref{alg:ONS}, we have
	$0\leq\inner{w_t-w'_{t+1}}{\nabla_t-m_t}\leq 2\norm{\nabla_t-m_t}_{A_t^{-1}}^2$ and also $\sumt \norm{\nabla_t-m_t}_{A_t^{-1}}^2= \bigO{\frac{r\ln(Tz_T/z_1)}{\eta}}$.
\end{lemma}
\begin{proof}
	For any $t$, define $F_x(w)=\inner{w}{x} + D_{\psi_t}(w, w'_t)$.
	Then, we have 
	$$w_t=\argmin_{w\in\calK}F_{m_t}(w) \quad\text{and}\quad w'_{t+1}=\argmin_{w\in\calK}F_{\nabla_t}(w).$$
	Moreover, $\nabla^2_wD_{\psi_t}(w, w'_t)=A_t$ is a constant matrix.
	Hence, by \pref{lem:stability} with $c=1$, $0\leq\inner{w_t-w'_{t+1}}{\nabla_t-m_t}\leq 2\norm{\nabla_t-m_t}_{A_t^{-1}}^2$.
	
	Define $A_t' = \eta\left(4z_1^2\cdot I+\sum_{s=1}^t (\nabla_s-m_s)(\nabla_s-m_s)^\top\right)$.
	Note that $\norm{\nabla_t-m_t}_2^2 \leq 4\norm{\ell_t-m_t}_2^2 \leq 4z_t^2$.
	Thus, $A_t \mgeq A'_t$.
	By similar arguments in \cite[Lemma 6]{koren2017affine}, we have
	\begin{align*}
		\sum_{t=1}^T\norm{\nabla_t - m_t}_{A_t^{-1}}^2 &=  \frac{1}{\eta}\sum_{t=1}^T \trace{A^{-1}_t(A'_t-A'_{t-1})} \leq \frac{1}{\eta}\sum_{t=1}^T \trace{(A'_t)^{-1}(A'_t-A'_{t-1})} \\
		&\leq \frac{1}{\eta}\ln\frac{|A'_T|}{|A'_0|} = \frac{1}{\eta}\ln\abr{I + \sumt\frac{(\nabla_t-m_t)(\nabla_t-m_t)^\top}{4z_1^2}} \\
		&= \bigO{\frac{r\ln\rbr{1+\sumt\frac{\norm{\ell_t-m_t}_2^2}{rz_1^2}}}{\eta}} = \bigO{\frac{r\ln (Tz_T/z_1)}{\eta}},
	\end{align*}
	where $r$ is rank of $\sumt(\ell_t-m_t)(\ell_t-m_t)^\top$. 
\end{proof}

To obtain the regret bound in \pref{thm:ONS}, we instantiate \master with the following set of experts:

\begin{equation}
\label{eq:ONS_base_alg}
\begin{split}
	\ONSbase &= \Big\{ (\eta_k, \calA_k): \forall k=(d_k, s_k) \in \left\{-\ceil{\log_2(dT)}, \ldots,  \ceil{\log_2D} \right\} \times[\ceil{\log_2T}], \eta_k = \tfrac{1}{64\cdot 2^{d_k+s_k}}, \\
	&\qquad\text{$\calA_k$ is \pref{alg:ONS} with $z_t=1$ for all $t$, }  \text{$\Omega = \calK\cap\{w:\norm{w}_2\leq 2^{d_k}\}$, and $\eta=3\eta_k$} \Big\}
\end{split}
\end{equation}

\begin{proof}[of \pref{thm:ONS}]
	We first assume $\norm{u}_2>\frac{1}{dT}$, and thus there exists $k_{\star}$ such that $\eta_{\kstar}\leq \min\left\{\frac{1}{192\cdot 2^{d_\kstar}}, \sqrt{\frac{r\ln T}{\sumt\inner{u}{\ell_t-m_t}^2}}\right\}\leq 2\eta_{\kstar}$, and $2^{d_{\kstar}-1}\leq \norm{u}_2\leq 2^{d_{\kstar}}$.
	Then by \pref{lem:ONS} with $64\cdot 3\eta_{\kstar}\cdot 2^{d_\kstar}\leq 1$:
	\begin{align*}
		\sumt \inner{w^{\kstar}_t-u}{\ell_t} &\leq \bigO{ \frac{r\ln T}{\eta_{\kstar}} + \norm{u}_2 + \eta_{\kstar}\sumt\inner{u}{\ell_t-m_t}^2 } - 33\eta_{\kstar}\sumt\inner{w^{\kstar}_t}{\ell_t-m_t}^2\\
		&= \tilO{ r\norm{u}_2 + \sqrt{r\sumt\inner{u}{\ell_t-m_t}^2} } - 33\eta_{\kstar}\sumt\inner{w^{\kstar}_t}{\ell_t-m_t}^2.
	\end{align*}
	Next, by \pref{thm:master} with $32\eta_k\left|\inner{w^k_t}{\ell_t-m_t}\right|\leq 32\eta_k\norm{w^k_t}_2 \leq 1$, $\sum_k\eta_k=\Theta(dT)$, $\sum_k\eta^2_k=\Theta(d^2T^2)$, and $\frac{\sum_k\eta_k^2}{\eta_{\kstar}^2}=\bigo{d^2T^2/\eta_{\kstar}^2}=\bigo{ d^2D^2T^4 }$, we have:
	\begin{align*}
		\sumt\inner{w_t-u}{\ell_t} &= \tilO{ r\norm{u}_2 + \sqrt{r\sumt\inner{u}{\ell_t-m_t}^2} + \frac{1}{\eta_{\kstar}} }\\
		&= \tilO{ r\norm{u}_2 + \sqrt{r\sumt\inner{u}{\ell_t-m_t}^2} }.
	\end{align*}
	When $\norm{u}_2 \leq \frac{1}{dT}\leq D$ (if $D<\frac{1}{dT}$, we achieve constant regret by picking $w_t$ arbitrarily), pick any $u'\in\calK$ such that $\norm{u'}_2 = \frac{1}{dT}$ (this is possible since $\bf{0}\in\calK$). Then:
	\begin{align*}
		\sumt\inner{w_t-u}{\ell_t} &= \sumt\inner{w_t-u'}{\ell_t} + \sumt\inner{u'-u}{\ell_t}\\
		&\leq \tilO{ r\norm{u'}_2 + \sqrt{r\sumt\inner{u'}{\ell_t-m_t}^2} + \norm{u'}_2T } = \tilO{1}.
	\end{align*}
	This finishes the proof.
\end{proof}

\subsection{Combining Gradient Descent}\label{app:GD}

For gradient descent type of bound, we use the optimistic gradient descent algorithm (OptGD) as the base algorithm, which achieves the following regret bound with learning rate $\eta$ (see~\citep[Lemma 3]{rakhlin2013optimization}):
\begin{align*}
	\sumt\inner{w_t-u}{\ell_t} \leq \frac{\norm{u}_2^2}{\eta} + \eta\sumt\norm{\ell_t-m_t}_2^2.
\end{align*}
To obtain the regret bound in \pref{thm:GD}, it suffices to instantiate \master with the following set of experts:
 
\begin{equation}
\label{eq:GD_base_alg}
\begin{split}
	\GDbase &= \Big\{ (\eta_k, \calA_k): \forall k=(d_k, s_k) \in \left\{-\ceil{\log_2T} ,\ldots, \ceil{\log_2D}\right\}\times[\ceil{\log_2T}], \eta_k = \tfrac{1}{32\cdot 2^{d_k+s_k}}, \\
	&\qquad\qquad\text{$\calA_k$ is OptGD with decision set }  \text{$\Omega = \calK\cap\{w:\norm{w}_2\leq 2^{d_k}\}$, and $\eta=4^{d_k}\eta_k$} \Big\}.
\end{split}
\end{equation}

\begin{proof}[of \pref{thm:GD}]
	We first assume $\norm{u}_2>\frac{1}{T}$, so that there exists $\kstar$ such that 
	\[
		\eta_{\kstar} \leq \min\left\{\frac{1}{64\cdot 2^{d_{\kstar}}}, \frac{1}{\norm{u}_2\sqrt{\sumt\norm{\ell_t-m_t}_2^2}} \right\} \leq 2\eta_{\kstar}, 
	\]
	and $2^{d_{\kstar}-1}\leq \norm{u}_2 \leq 2^{d_{\kstar}}$.
	By the regret guarantee of OptGD, we have:
	\begin{align*}
		\sumt\inner{w^{\kstar}_t - u}{\ell_t} \leq \frac{\norm{u}_2^2}{4^{d_{\kstar}}\eta_{\kstar}} + 4^{d_{\kstar}}\eta_{\kstar}\sumt\norm{\ell_t-m_t}_2^2 = \bigO{ \norm{u}_2 + \norm{u}_2\sqrt{\sumt\norm{\ell_t-m_t}_2^2} }.
	\end{align*}
	Next, by \pref{thm:master} with $32\eta_k\left|\inner{w^k_t}{\ell_t-m_t}\right|\leq 32\eta_k\norm{w^k_t}_2 \leq 1$, $\sum_k\eta_k=\Theta(T)$, $\sum_k\eta^2_k=\Theta(T^2)$, and $\frac{\sum_k\eta_k^2}{\eta_{\kstar}^2}=\bigo{ D^2T^3 }$, we have:
	\begin{align*}
		\sumt\inner{w_t-u}{\ell_t} &= \tilO{\norm{u}_2 + \norm{u}_2\sqrt{\sumt\norm{\ell_t-m_t}_2^2} + \eta_{\kstar}\sumt\inner{w^{\kstar}_t}{\ell_t-m_t}^2 }\\
		&= \tilO{ \norm{u}_2 + \norm{u}_2\sqrt{\sumt\norm{\ell_t-m_t}_2^2} }.
	\end{align*}
	When $\norm{u}_2\leq\frac{1}{T}$, pick any $u'\in\calK$ such that $\norm{u'}_2=\frac{1}{T}$, then:
	\begin{align*}
		\sumt\inner{w_t-u}{\ell_t} &= \sumt\inner{w_t-u'}{\ell_t} + \sumt\inner{u'-u}{\ell_t}\\
		&\leq \tilO{ \norm{u}_2 + \norm{u}_2\sqrt{\sumt\norm{\ell_t-m_t}_2^2} + \norm{u}_2 } = \tilO{1}.
	\end{align*}
	This finishes the proof.
\end{proof}
 
\subsection{Combining AdaGrad}\label{app:AdaGrad}

We first introduce the base algorithm \pref{alg:AdaGrad}, which is a variant of the AdaGrad algorithm with predictor $m_t$ incorporated. It guarantees the following.

\setcounter{AlgoLine}{0}
\begin{algorithm}[t]
    \caption{Optimistic AdaGrad}
    \label{alg:AdaGrad}
    \textbf{Parameters:} learning rate $\eta, \eta'>0$, $w'_1=\vec{0}$.
    
    \textbf{Define:}
    \begin{align*}
        c_t(w)& = \inner{w}{\ell_t}+16\eta'\inner{w}{\ell_t - m_t}^2 \\
        \nabla_t &= \nabla c_t(w_t) = \ell_t + 32\eta'\inner{w_t}{\ell_t-m_t}(\ell_t-m_t)\\
        \psi_t(w) &= \frac{1}{2}\norm{w}_{A_t}^2, \qquad \text{where\ }A_t\triangleq \frac{1}{\eta}(I + \calG_t)^{1/2},\quad \calG_t = \sum_{s=1}^t(\nabla_t-m_t)(\nabla_t-m_t)^\top.
    \end{align*}
    
    \For{$t=1$,\ldots, T}{
		Receive prediction $m_t$.
    
    		Compute $w_t = \argmin_{w\in\Omega}\left\{\inner{w}{\sum_{s=1}^{t-1}\nabla_s+m_t} + \psi_{t-1}(w)\right\}$.
    
    		Play $w_t$ and receive $\ell_t$.
    }
\end{algorithm}

\begin{theorem}
	Define $A'_t=(I+ \sum_{s=1}^t(\ell_s-m_s)(\ell_s-m_s)^\top)^{1/2}$.
	Assume $64\eta'|\inner{w_t}{\ell_t-m_t}|\leq 1$ for all $t$, and $\eta'\leq \eta/\norm{u}^2_{A'_T}$.
	\pref{alg:AdaGrad} ensures for any $u\in\Omega$,
	\begin{align*}
		\sumt\inner{w_t-u}{\ell_t} = \bigO{\eta\trace{\calL_T^{1/2}} + \frac{u^{\top}(I+\calL_T)^{1/2}u}{\eta} } - 16\eta'\sumt\inner{w_t}{\ell_t-m_t}^2.
	\end{align*}
\end{theorem}
\begin{proof}
	For any $t$, define $F_x(w)=\inner{w}{x} + \psi_{t-1}(w)$.
	Note that $w_t=\argmin_{w\in\calK}F_{\sum_{s=1}^{t-1}\nabla_s+m_t}(w)$, and denote $w'_t=\argmin_{w\in\calK}F_{\sum_{s=1}^t\nabla_s}(w)$.
	Moreover, $\nabla^2\psi_{t-1}(w)=A_{t-1}$ is a constant matrix.
	Hence, by \pref{lem:stability} with $c=1$, $\inner{w_t-w'_t}{\nabla_t-m_t}\leq 2\norm{\nabla_t-m_t}_{A_{t-1}^{-1}}^2$, and for any $u\in\Omega$ we have:
	\begin{align*}
		\sumt\inner{w_t-u}{\nabla_t} &= \sumt\inner{w_t - w'_t}{\nabla_t-m_t} + \inner{w_t-w'_t}{m_t} + \inner{w'_t-u}{\nabla_t}\\
		&\leq 2\sumt \norm{\nabla_t-m_t}_{A_{t-1}^{-1}}^2 + \sumt\inner{w_t-w'_t}{m_t} + \inner{w'_t-u}{\nabla_t}.
	\end{align*}
	We prove by induction that for any $\tau, u\in\Omega$:
	\begin{align*}
		\sum_{t=1}^{\tau}\inner{w_t-w'_t}{m_t} + \inner{w'_t}{\nabla_t} \leq \sum_{t=1}^{\tau}\inner{u}{\nabla_t} + \psi_{\tau-1}(u).
	\end{align*}
When $\tau=1$, it suffices to show:
	\begin{align*}
		\inner{w_1-w'_1}{m_1} + \inner{w'_1}{\nabla_1} \leq \inner{w'_1}{\nabla_1} + \psi_0(w'_1).
	\end{align*}
	This is clearly true since $\inner{w_1}{m_1} \leq \inner{w_1}{m_1}+\psi_0(w_1)\leq \inner{w'_1}{m_1}+\psi_0(w'_1)$. Now suppose the result is true for $\tau=T$, then for $\tau=T+1$:
	\begin{align*}
		&\sum_{t=1}^{T+1}\inner{w_t-w'_t}{m_t} + \inner{w'_t}{\nabla_t}\\
		&\leq \inner{w_{T+1}}{\sumt\nabla_t} + \psi_{T-1}(w_{T+1}) + \inner{w_{T+1}-w'_{T+1}}{m_{T+1}} + \inner{w'_{T+1}}{\nabla_{T+1}} \tag{induction step for $\tau=T$ with $u=w_{T+1}$}\\
		&\leq \inner{w'_{T+1}}{\sumt\nabla_t+m_{T+1}} + \psi_T(w'_{T+1}) - \inner{w'_{T+1}}{m_{T+1}} + \inner{w'_{T+1}}{\nabla_{T+1}} \tag{by $\psi_{T-1}(w)\leq\psi_T(w)$, and $F_{\sum_{t=1}^T\nabla_t+m_{T+1}}(w_{T+1})\leq F_{\sum_{t=1}^T\nabla_t+m_{T+1}}(w'_{T+1})$}\\
		&= \inner{w'_{T+1}}{\sum_{t=1}^{T+1}\nabla_t} + \psi_T(w'_{T+1}) \leq \inner{u}{\sum_{t=1}^{T+1}\nabla_t} + \psi_T(u),
	\end{align*}
	for any $u\in\Omega$ by the definition of $w'_{T+1}$. Therefore, by \citep[Theorem 7]{cutkosky2020better}, we have:
	\begin{align*}
		\sumt\inner{w_t-u}{\nabla_t} &\leq 2\sumt\norm{\nabla_t-m_t}_{A_{t-1}^{-1}}^2 + \psi_{T-1}(u) \leq 2\sumt\norm{\nabla_t-m_t}_{A_{t-1}^{-1}}^2 + \frac{u^{\top}(I + \calG_T)^{1/2}u}{\eta}\\
		&= \bigO{ \eta\trace{\calG_T^{1/2}} + \frac{u^{\top}(I+\calG_T)^{1/2}u}{\eta} } = \bigO{ \eta\trace{\calL_T^{1/2}} + \frac{u^{\top}(I+\calL_T)^{1/2}u}{\eta} }.
	\end{align*}
	The reasoning of the last equality is as follows: note that $\nabla_t-m_t=(1+32\eta'\inner{w_t}{\ell_t-m_t})(\ell_t-m_t)$ has the same direction as $\ell_t-m_t$.
	Thus by assumption on $\eta'$, $\calG_t \mleq \frac{3}{2}\calL_t$.
	Finally, note that $c_t$ is a convex function. Therefore, $\sumt c_t(w_t)-c_t(u)\leq \sumt\inner{w_t-u}{\nabla_t}$. Reorganizing terms, we get:
	\begin{align*}
		&\sumt\inner{w_t-u}{\ell_t}\\
		&\leq \bigO{ \eta\trace{\calL_T^{1/2}} + \frac{u^{\top}(I+\calL_T)^{1/2}u}{\eta} } - 16\eta'\sumt\inner{w_t}{\ell_t-m_t}^2 + 16\eta'\sumt\inner{u}{\ell_t-m_t}^2.
	\end{align*}
	By $\eta'\leq \eta/\norm{u}^2_{A'_T}$ (note that $\norm{u}^2_{A'_T}=u^\top(I+\calL_T)^{1/2}u$), we have:
	\begin{align*}
		\eta'\sumt\inner{u}{\ell_t-m_t}^2 \leq \eta'\sumt\norm{u}_{A'_T}^2\norm{\ell_t-m_t}_{(A'_{t-1})^{-1}}^2 = \bigO{ \eta\trace{\calL_T^{1/2}} }.
	\end{align*}
	Therefore,
	\begin{align*}
		\sumt\inner{w_t-u}{\ell_t} = \bigO{\eta\trace{\calL_T^{1/2}} + \frac{u^{\top}(I+\calL_T)^{1/2}u}{\eta} } - 16\eta'\sumt\inner{w_t}{\ell_t-m_t}^2.
	\end{align*}
\end{proof}

Now we instantiate \master with the following set of experts to obtain the desired bound in \pref{thm:AdaGrad}.

\begin{equation}
\label{eq:AdaGrad_base_alg}
\begin{split}
	\AdaGradbase &= \Big\{ (\eta_k, \calA_k): \forall k=(d_k, t_k, l_k) \in \calS_{\text{\rm AG}}, \\
	&\qquad\qquad\eta_k = \tfrac{1}{64\cdot 2^{d_k+t_k}}, \text{$\calA_k$ is \pref{alg:AdaGrad} with decision set }  \text{$\Omega = \calK\cap\{w:\norm{w}_2\leq 2^{d_k}\}$, }\\
	&\qquad\qquad\text{$\eta'=2\eta_k$ and $\eta= 2^{l_k+1}\eta_k$} \Big\},
\end{split}
\end{equation}
where $\calS_{\text{\rm AG}}=\{-\ceil{\log_2T} ,\ldots, \ceil{\log_2D}\}\times[\ceil{\log_2(dT)}]\times\{-\ceil{\log_2T},\ldots,\ceil{\log_2(2D^2T)}\}$. \\

\begin{proof}[of \pref{thm:AdaGrad}]
	First assume $\norm{u}_2>\frac{1}{T}$, so that there exists $\kstar$ such that:
	\begin{align*}
		2^{d_{\kstar}-1} &\leq \norm{u}_2 \leq 2^{d_{\kstar}}, \eta_{\kstar} \leq \min\left\{\frac{1}{128\cdot 2^{d_\kstar}}, \frac{1}{\sqrt{\norm{u}_{(I+\calL_T)^{1/2}}^2\trace{\calL_T^{1/2}}}} \right\} \leq 2\eta_{\kstar},
	\end{align*}
	and $2^{l_k-1} \leq u^\top(I+\calL_T)^{1/2}u \leq 2^{l_k}$.
	Note that $64\eta'|\inner{w^{\kstar}_t}{\ell_t-m_t}|\leq 64\eta'\norm{w^{\kstar}_t}_2\leq 1$, and $\norm{u}_{A'_T}^2\eta'\leq 2^{l_{\kstar}}\cdot 2\eta_{\kstar}= \eta$.
	Hence, by the regret guarantee of \pref{alg:AdaGrad}, we have:
	\begin{align*}
		\sumt\inner{w^{\kstar}_t-u}{\ell_t} &\leq \bigO{2^{l_k+1}\eta_{\kstar}\trace{\calL_T^{1/2}} + \frac{u^{\top}(I+\calL_T)^{1/2}u}{2^{l_k+1}\eta_{\kstar}} } - 32\eta_{\kstar}\sumt\inner{w^{\kstar}_t}{\ell_t-m_t}^2\\
		&\leq \bigO{ \norm{u} + \sqrt{(u^\top(I+\calL_T)^{1/2}u)\trace{\calL_T^{1/2}}} } - 32\eta_{\kstar}\sumt\inner{w^{\kstar}_t}{\ell_t-m_t}^2.
	\end{align*}
	Next, by \pref{thm:master} with $32\eta_k\left|\inner{w^k_t}{\ell_t-m_t}\right|\leq 32\eta_k\norm{w^k_t}_2 \leq 1$, $\sum_k\eta_k=\Theta(T)$, $\sum_k\eta^2_k=\Theta(T^2)$, and $\frac{\sum_k\eta_k^2}{\eta_{\kstar}^2}=\bigo{ d^2D^2T^4 }$, we have:
	\begin{align*}
		\sumt\inner{w_t-u}{\ell_t} = \tilO{ \norm{u}_2 + \sqrt{(u^\top(I+\calL_T)^{1/2}u)\trace{\calL_T^{1/2}}} }.
	\end{align*}
	When $\norm{u}_2\leq\frac{1}{T}$, pick any $u'\in\calK$ such that $\norm{u'}_2=\frac{1}{T}$, then:
	\begin{align*}
		\sumt\inner{w_t-u}{\ell_t} &= \sumt\inner{w_t-u'}{\ell_t} + \sumt\inner{u'-u}{\ell_t}\\
		&\leq \tilO{ \norm{u'}_2 + \sqrt{(u'^\top(I+\calL_T)^{1/2}u')\trace{\calL_T^{1/2}}} + \norm{u'}_2 } = \tilO{1}.
	\end{align*} 
	This finishes the proof.
\end{proof}
 
\subsection{Combining MetaGrad's base algorithm}\label{app:MetaGrad}
\setcounter{AlgoLine}{0}
\begin{algorithm}[t]
    \caption{MetaGrad}
    \label{alg:MetaGrad}
    \textbf{Parameters:} learning rate $\eta>0$, $w'_1=\vec{0}$.
    
    \textbf{Define:}
    \begin{align*}
        c_t(w)& = \inner{w}{\ell_t}+16\eta\inner{w-\bar{w}_t}{\ell_t - m_t}^2 \\
        \nabla_t&=\nabla c_t(w_t)=\ell_t + 32\eta\inner{w_t-\bar{w}_t}{\ell_t-m_t}(\ell_t-m_t) \\
        \psi_t(w) &= \frac{1}{2}\norm{w}_{A_t}^2, \qquad \text{where\ }A_t\triangleq \eta\rbr{ 8I + \sum_{s=1}^{t-1} (\nabla_s-m_s)(\nabla_s-m_s)^\top }.
    \end{align*}
    
    \For{$t=1$,\ldots, T}{
		Receive prediction $m_t$.
    
    		Play $w_t = \argmin_{w\in\calK}\left\{\inner{w}{m_t} + D_{\psi_{t}}(w,w_t')\right\}$.
    
    		Receive $\ell_t$ and $\bar{w}_t$.
    
        Compute $w_{t+1}' = \argmin_{w\in\calK}\left\{\inner{w}{\nabla_t} + D_{\psi_t}(w, w_t')\right\}$.
    }
\end{algorithm}

We first present the MetaGrad base algorithm (\pref{alg:MetaGrad}) and its regret guarantee below (note that the algorithm receives $\bar{w}_t$ at the end of round $t$, which will eventually be set to the master's prediction in our construction).
\begin{lemma}
	\label{lem:MetaGrad}
	Assume $64\eta D\leq 1$.
	\pref{alg:MetaGrad} ensures:
	\begin{align*}
		\sumt\inner{w_t-u}{\ell_t} \leq \bigO{ \norm{u}_2 + \frac{r\ln T}{\eta} + \eta\sumt\inner{u-\bar{w}_t}{\ell_t-m_t}^2 } - 10\eta\sumt\inner{w_t-\bar{w}_t}{\ell_t-m_t}^2.
	\end{align*}
\end{lemma}
\begin{proof}
	By \pref{lem:oomd} and \pref{lem:olo stabability} with $z_t=1$ for all $t$, we have:
	\begin{align*}
		&\sum_{t=1}^T\inner{w_t - u}{\nabla_t}\\
		&\leq \sum_{t=1}^T\inner{w_t - w'_{t+1}}{\nabla_t - m_t} + D_{\psi_t}(u, w'_t) - D_{\psi_t}(u, w'_{t+1})\\
		&\leq 2\sum_{t=1}^T\norm{\nabla_t -m_t}_{A_t^{-1}}^2 + D_{\psi_1}(u, w'_1) + \sum_{t=1}^{T-1}D_{\psi_{t+1}}(u, w'_{t+1}) - D_{\psi_t}(u, w'_{t+1})\\
		&\leq \bigO{\frac{r\ln T}{\eta} + \eta \norm{u}_2^2} + \sum_{t=1}^{T-1}D_{\psi_{t+1}}(u, w'_{t+1}) - D_{\psi_t}(u, w'_{t+1}).
	\end{align*}
	Note that $\eta\norm{u}_2^2=\bigO{\norm{u}_2}$.
	Moreover,
	\begin{align*}
		&\sum_{t=1}^{T-1}D_{\psi_{t+1}}(u, w'_{t+1}) - D_{\psi_t}(u, w'_{t+1})\\
		&= \frac{\eta}{2}\sum_{t=1}^{T-1}\inner{u-w'_{t+1}}{\nabla_t - m_t}^2\\
		&\leq \eta\sum_{t=1}^{T-1}\inner{u-w_t}{\nabla_t-m_t}^2 + \eta\sum_{t=1}^{T-1}\inner{w_t-w'_{t+1}}{\nabla_t-m_t}^2\\
		&\leq 3\eta\sum_{t=1}^{T-1}\inner{u-w_t}{\ell_t-m_t}^2 + \bigO{\frac{r\ln T}{\eta}},
	\end{align*}
	where the last step is by $0\leq \eta\inner{w_t-w'_{t+1}}{\nabla_t-m_t}\leq 3\eta D= \order(1)$ and \pref{lem:olo stabability}.
	Since $c_t(w)$ is convex in $w$, we have $\sum_{t=1}^Tc_t(w_t) - c_t(u) \leq \sum_{t=1}^T\inner{w_t - u}{\nabla_t}$. Re-organzing terms, we have:
	\begin{align*}
		\sumt\inner{w_t-u}{\ell_t}&\leq \bigO{\frac{r\ln T}{\eta}+\norm{u}_2} + 3\eta\sumt\inner{u-w_t}{\ell_t-m_t}^2 \\
		&\qquad\qquad + 16\eta\sumt\inner{u-\bar{w}_t}{\ell_t-m_t}^2 - 16\eta\sumt\inner{w_t-\bar{w}_t}{\ell_t-m_t}^2\\
		&\leq \bigO{\frac{r\ln T}{\eta} + \norm{u}_2 + \eta\sumt\inner{u-\bar{w}_t}{\ell_t-m_t}^2 } - 10\eta\sumt\inner{w_t-\bar{w}_t}{\ell_t-m_t}^2.
	\end{align*}
\end{proof}

Then, we instantiate \master with the following set of experts to obtain the desired bound in \pref{thm:MetaGrad}.
\begin{equation}
\label{eq:MetaGrad_base_alg}
\begin{split}
	\MetaGradbase &= \Big\{ (\eta_k, \calA_k): \forall k \in [\ceil{\log_2(2DT)}], \eta_k = \tfrac{1}{64D\cdot 2^{k}}, \\
	&\qquad\qquad\text{$\calA_k$ is \pref{alg:MetaGrad} with $\bar{w}_t=w_t$ for all $t$ and $\eta=4\eta_k$} \Big\}.
\end{split}
\end{equation}

\begin{proof}[of \pref{thm:MetaGrad}]
	There exists $k_{\star}$ such that $\eta_{\kstar}\leq \min\left\{\frac{1}{256D}, \sqrt{\frac{r\ln T}{\sumt\inner{u-w_t}{\ell_t-m_t}^2}}\right\}\leq 2\eta_{\kstar}$. Then by \pref{lem:MetaGrad} with $64\cdot 4\eta_{\kstar}D\leq 1$:
	\begin{align*}
		&\sumt \inner{w^{\kstar}_t-u}{\ell_t}\\
		&\leq \bigO{ \frac{r\ln T}{\eta_{\kstar}} + \norm{u}_2 + \eta_{\kstar}\sumt\inner{u-w_t}{\ell_t-m_t}^2 } - 40\eta_{\kstar}\sumt\inner{w^{\kstar}_t-w_t}{\ell_t-m_t}^2\\
		&= \tilO{ rD + \sqrt{r\sumt\inner{u-w_t}{\ell_t-m_t}^2} } - 40\eta_{\kstar}\sumt\inner{w^{\kstar}_t-w_t}{\ell_t-m_t}^2.
	\end{align*}
	Next, by \pref{thm:master} with $32\eta_k|g_{t,k}-h_{t,k}|=32\eta_k\left|\inner{w^k_t-w_t}{\ell_t-m_t}\right|\leq 64\eta_kD \leq 1$, $\sum_k\eta_k=\Theta(1/D)$, $\sum_k\eta^2_k=\Theta(1/D^2)$, and $\frac{\sum_k\eta_k^2}{\eta_{\kstar}^2}=\bigo{ D^4T^2 }$, we have:
	\begin{align*}
		\sumt\inner{w_t-u}{\ell_t} &= \tilO{ rD + \sqrt{r\sumt\inner{u-w_t}{\ell_t-m_t}^2} + \frac{1}{\eta_{\kstar}} }\\
		&= \tilO{ rD + \sqrt{r\sumt\inner{u-w_t}{\ell_t-m_t}^2} }.
	\end{align*}
	This completes the proof.
\end{proof}

\subsection{Extensions to unconstrained learning and unknown Lipschitzness}
\label{app:unconstrained}





In this subsection, we present general ideas on extending our OLO results to the setting with an unconstrained decision set, unknown Lipschitzness, or both. We focus on $\sqrt{r\sum_t \inner{u}{\ell_t - m_t}^2}$ type of bound  and omit the details for the others for simplicity.

\subsubsection{Unconstrained learning with known Lipschitzness}
We first consider the case where $D=\infty$ and $\max_t \max\{\norm{\ell_t}_2, \norm{\ell_t-m_t}_2\} \leq 1$.
We argue that in this case we can simply assume that $\norm{u}_2\leq 2^T$, so that we only need to maintain $\bigO{T}$ experts.
Suppose the assumption does not hold and $T < \log_2 \norm{u}_2$.
Then, by constraining $\norm{w_t}_2\leq  2^T$, we have: $\sumt\inner{w_t-u}{\ell_t}\leq 2T\norm{u}_2 < 2\norm{u}_2\log_2\norm{u}_2=\tilO{\norm{u}_2}$.
Therefore, running the algorithm in \pref{thm:ONS} assuming the diameter is $2^T$, we obtain the same bound as before:
\begin{align*}
	\reg(u) = \tilO{r\norm{u}_2 + \sqrt{r\sumt\inner{u}{\ell_t-m_t}^2}}.
\end{align*}
Note that when $m_t=0$, the bound we obtained has the same order as that in \citep[Theorem 8]{cutkosky2018black}.

\subsubsection{Constrained learning with unknown Lipschitzness}
Next, we consider the case where $D<\infty$ and $\max_t\norm{\ell_t-m_t}_2$ is unknown.
We can handle this by simply applying our master with unknown loss range (\pref{alg:unknown_range}) with the following expert set generator:

\begin{equation*}
\begin{split}
	\calE_{\text{\rm ONSUL}}(B_0) &= \Big\{ (\eta_k, \calA_k): \forall k \in [N], \eta_k = \tfrac{1}{192DB_0 2^k}, \text{$\calA_k$ is \pref{alg:ONS} } \\
	& \text{with $z_t=B_{t-1} = \max_{0\leq s<t}\norm{\ell_s-m_s}$ for all $t$, $\Omega = \calK$, and $\eta=3\eta_k$} \Big\}.
\end{split}
\end{equation*}
and $\Lambda_t=\Delta_{ \calS_{\text{\rm ONSUL}}(t) }$, where $N=\lceil\log_2T^2 \rceil, \calS_{\text{\rm ONSUL}}(t)=\left\{k\in[N]: \tfrac{1}{192DB_0 2^k}\leq \frac{1}{192DB_{t-1}} \right\}$.

\begin{theorem}
	\label{thm:known_d_unknown_range}
	Let $\max_t \norm{\ell_t - m_t}_2$ be unknown, $r \leq d$ be the rank of $\calL_T =\sumt(\ell_t-m_t)(\ell_t-m_t)^\top$.
	\pref{alg:unknown_range} with expert set generator $\calE_{\text{\rm ONSUL}}$ and $\Lambda_t=\Delta_{ \calS_{\text{\rm ONSUL}}(t) }$ ensures for all $u\in\simplex$,
	\[
		\forall u \in \calK, \;\;  \reg(u) = \tilO{rDB + \sqrt{r\sumt\inner{u}{\ell_t-m_t}^2}}.
	\]
\end{theorem}
\begin{proof}
	We first show that when there is no restart before episode $t$, we obtain the desired regret bound. The assumption implies that $\frac{B_{T-1}}{B_0}\leq T$, and thus $\frac{1}{\max\{192, T\}DB_{T-1}} \geq \frac{1}{192DB_02^N}$. Therefore, there exists $\kstar$ such that
	\begin{align*}
		\eta_{\kstar} \leq \min\left\{\frac{1}{192DB_{T-1}}, \sqrt{\frac{r\ln(TB_{T-1}/B_0)}{\sumt\inner{u}{\bar{\ell}_t-m_t}^2}} \right\} \leq 2\eta_{\kstar}.
	\end{align*}
	Hence, by \pref{lem:ONS} with $64\cdot 3\eta_{\kstar}DB_{T-1}\leq 1$, we have:
	\begin{align*}
		\sumt\inner{w^{\kstar}_t-u}{\bar{\ell}_t} &= \bigO{ \frac{r\ln T}{\eta_{\kstar}} + DB + \eta_{\kstar}\sumt\inner{u}{\bar{\ell}_t-m_t}^2 } - 33\eta_{\kstar}\sumt\inner{w^{\kstar}_t}{\ell_t-m_t}^2\\
		&= \tilO{ rDB + \sqrt{r\sumt\inner{u}{\bar{\ell}_t-m_t}^2} } - 33\eta_{\kstar}\sumt\inner{w^{\kstar}_t}{\ell_t-m_t}^2.
	\end{align*}
	By \pref{thm:master} with $32\eta_k\left|\inner{w^k_t}{\bar{\ell}_t-m_t}\right|\leq 32\eta_kDB_{t-1} \leq 1$ for any $k\in \calS_{\text{\rm ONSUL}}(t)$, $\sum_k\eta_k=\Theta(\frac{1}{DB_0})$, $\sum_k\eta^2_k=\Theta(\frac{1}{D^2B_0^2})$, and $\frac{\sum_k\eta_k^2}{\eta_{\kstar}^2}=\bigo{\eta_1^2/\eta_{\kstar}^2}=\bigo{ T^4 }$, we have:
	\begin{align*}
		\sumt\inner{w_t-u}{\bar{\ell}_t} = \tilO{ rDB + \sqrt{r\sumt\inner{u}{\bar{\ell}_t-m_t}^2} }.
	\end{align*}
	Moreover, note that,
	\begin{align*}
		\sumt\inner{w_t-u}{\ell_t - \bar{\ell}_t} &\leq 2\sumt D\norm{\ell_t-\bar{\ell_t}}_2 \leq 2\sumt D\frac{B_t-B_{t-1}}{B_t}\norm{\ell_t-m_t}_2\\
		&\leq 2D\sumt (B_t - B_{t-1}) \leq 2DB.
	\end{align*}
	Therefore, by $\inner{u}{\bar{\ell_t}-m_t}^2 = (1-\frac{B_{t-1}}{B_t})^2\inner{u}{\ell_t-m_t}^2 \leq \inner{u}{\ell_t-m_t}^2$,
	\begin{align*}
		\sumt\inner{w_t - u}{\ell_t} &= \sumt\inner{w_t-u}{\bar{\ell}_t} + \sumt\inner{w_t-u}{\ell_t-\bar{\ell}_t}\\
		&= \tilO{ rDB + \sqrt{r\sumt\inner{u}{\ell_t-m_t}^2} }.
	\end{align*}
	Finally, we assume there are at least one restarts.
	Following similar analysis in the proof of \pref{thm:unknown_range}, we consider regret in the following three intervals: $[1,\tau_1], (\tau_1,\tau_2]$, and $(\tau_2, T]$.
	The regret in $[1, \tau_1]$ is bounded by $B$ according to \pref{eq:first restart interval}.
	By $B_{\tau_2-1}/B_{\tau_1}\leq T, B_{T-1}/B_{\tau_2}\leq T$, we have:
	\begin{align*}
		\reg^{(\tau_1, \tau_2]}(u) &= \bigO{ rDB + \sqrt{r\sum_{t\in (\tau_1, \tau_2]}\inner{u}{\bar{\ell}_t-m_t}^2} }\\
		\reg^{(\tau_2, T]}(u) &= \bigO{ rDB + \sqrt{r\sum_{t\in (\tau_2, T]}\inner{u}{\bar{\ell}_t-m_t}^2} }.
	\end{align*}
	Summing the regret in three intervals and applying the Cauchy-Schwarz inequality, we get the desired result.
\end{proof}

\subsubsection{Unconstrained learning with unknown Lipschitzness}
Finally, we consider the case where $D=\infty$ and $\max_t\norm{\ell_t-m_t}_2$ is unknown. \citet{cutkosky2019artificial,mhammedi2020lipschitz} show that to obtain $\tilO{\sqrt{T}}$ regret, it is sufficient to control the diameter of decision set to be of order $\tilO{\sqrt{T}}$.
Specifically, they set the size of the decision set to be $\sqrt{\max_{s\leq t}\sum_{s'=1}^{s}\norm{\ell_{s'}}_2/G_s}$ in episode $t$, where $G_t=\max_{s\leq t}\norm{\ell_s}_2$.
To bound the regret when the comparator is not in the decision set, they make use of a reduction to constrained domain~\citep{cutkosky2018black}.
However, their reduction is not directly applicable in our case, since the reduction modifies the loss function and ruins the data-dependent bound.
There is a follow up work~\citep{cutkosky2020better} achieving the bound $\tilO{\sqrt{r\sumt\inner{u}{\ell_t}^2}}$ under constrained domain by adapting to time-dependent norms.
However, it is not obvious how to incorporate predictor $m_t$ into their algorithm.

Here, we take a different route: we search over the appropriate constraint of the decision set with doubling trick: if in episode $t$ we find that $\sqrt{\sum_{s=1}^t\norm{\ell_s}_2/G_t}>D_t$, where $D_t$ is the diameter of decision set in episode $t$,
then, we let $D_{t+1}=2\sqrt{\sum_{s=1}^t\norm{\ell_s}_2/G_t}$, and restart the algorithm with the new decision set.
Otherwise we let $D_{t+1}=D_t$.
The number of restart is $\bigO{\log_2T}$ since $\max_{s\leq t}\sum_{s'=1}^{s}\norm{\ell_{s'}}_2/G_s \leq T$.
We summarize our algorithm as a new variant of \master in \pref{alg:unknown-D}.

\setcounter{AlgoLine}{0}
\begin{algorithm}
	\caption{\master with unknown loss range and unbounded diameter}
	\label{alg:unknown-D}
	\textbf{Input:} An expert set generator $\calE$ that takes diameter and initial scale as input, initial scale $B_0$.
	
	\textbf{Initialization:} $D_1=1$. Initialize $\calA$ as an instance of \pref{alg:unknown_range} with input $\calE(D_1, \cdot)$ and $B_0$.
	
	\For{$t=1,\ldots,T$}{
		Execute $\calA$ for episode $t$.
		
		\If{$D_t < \sqrt{\sum_{s=1}^t\frac{\norm{\ell_s}_2}{G_t}}$}{
			$D_{t+1}=2\sqrt{\sum_{s=1}^t\frac{\norm{\ell_s}_2}{G_t}}$.
		
			Initialize $\calA$ as an instance of \pref{alg:unknown_range} with input $\calE(D_{t+1}, \cdot)$ and $B_t$.
		}
		\Else{
			$D_{t+1}=D_t$.
		}
	}
\end{algorithm}

Now we show how to extend the regret bound of ONS to the setting with unconstrained diameter and unknown Lipschitzness.
\begin{theorem}
	Define the expert set generator:
	\begin{equation*}
	\begin{split}
		\calE_{\text{\rm ONSULD}}(D, B_0) &= \Big\{ (\eta_k, \calA_k): \forall k \in [N], \eta_k = \tfrac{1}{192DB_0 2^k}, \text{$\calA_k$ is \pref{alg:ONS} } \\
		&\qquad\qquad \text{with $z_t=B_{t-1} = \max_{0\leq s<t}\norm{\ell_s-m_s}$ for all $t$, $\Omega = \calK\cap\{w: \norm{w}_2\leq D\}$, }\\
		&\qquad\qquad\text{and $\eta=3\eta_k$} \Big\}.
	\end{split}
	\end{equation*}
	
	Then, \pref{alg:unknown-D} with input $\calE_{\text{\rm ONSULD}}, B_0$, ensures
	\begin{align*}
		\reg(u)=\tilO{ \sqrt{r\sumt\inner{u}{\ell_t-m_t}^2} + rB\sqrt{\max_{t\leq T}\sum_{s=1}^{t}\norm{\ell_{s}}_2/G_t} + G_T\norm{u}_2^3 },
	\end{align*}
	where $G_T = \max_{t\leq T}\norm{\ell_t}$, $B = \max\{B_0, \max_{t\leq T}\norm{\ell_t-m_t}\}$.
\end{theorem}
\begin{proof}
	We split $T$ episodes into $M$ intervals $I_{1:M}$, where the last episode of $I_m$ (denote by $t_m$) either equals to $T$ or $D_{t_m+1}\neq D_{t_m}$.
	Define projection function $f(u, D) = \min\left\{1, \frac{D}{\norm{u}_2}\right\}u$.
	Then, the regret is bounded as follows (note that $D_t=D_{t_m}$ for all $t\in I_m$):
	\begin{align*}
		\sumt\inner{w_t-u}{\ell_t} &= \sum_{m=1}^M\sum_{t\in I_m}\inner{w_t-f(u, D_{t_m})}{\ell_t} + \sumt\inner{f(u, D_t)-u}{\ell_t}.
	\end{align*}
	For the first term, by \pref{thm:known_d_unknown_range} with $\inner{f(u, D)}{\ell_t-m_t}^2\leq \inner{u}{\ell_t-m_t}^2$ for any $D>0$, and $M=\bigO{\log_2T}$, we obtain
	\begin{align*}
		\sum_{m=1}^M\sum_{t\in I_m}\inner{w_t-f(u, D_{t_m})}{\ell_t} &=\bigO{ \sum_{m=1}^M rD_{t_m}B + \sqrt{r\sum_{t\in I_m}\inner{u}{\ell_t-m_t}^2} }\\
		&= \tilO{rD_TB + \sqrt{r\sumt\inner{u}{\ell_t-m_t}^2} }.
	\end{align*}
	For the second term, denote by $t_\star$ the last episode such that $u\neq f(u, D_{t_\star})$.
	Then, $\norm{u}_2\geq \sqrt{\sum_{t=1}^{t_\star-1}\norm{\ell_t}_2/G_{t_\star-1}}$, $\norm{u}_2 \geq \norm{f(u, D_t)}_2$ for $t\leq t_\star$, and
	\begin{align*}
		\sumt\inner{f(u, D_t)-u}{\ell_t} &= \sum_{t=1}^{t_\star}\inner{f(u, D_t)-u}{\ell_t} \leq 2\sum_{t=1}^{t_\star-1}\norm{u}_2\norm{\ell_t}_2 + 2\norm{u}_2G_T\\
		&\leq 2\norm{u}_2G_T\sum_{t=1}^{t_\star-1}\frac{\norm{\ell_t}_2}{G_{t_\star-1}} + 2\norm{u}_2G_T \leq 2G_T\norm{u}_2^3+2\norm{u}_2G_T.
	\end{align*}
\end{proof}